\documentclass[runningheads,a4paper]{llncs}

\usepackage{comment}

\usepackage{graphicx}

\usepackage{times}
\usepackage{authblk}

\usepackage{amssymb,amsmath,latexsym}

\usepackage[usenames]{color}
\usepackage[all, knot]{xy}
\usepackage{graphicx}
\xyoption{arc}

\newcommand{\N}{\mbox{$I\!\!N$}}

\newcommand{\Outfamily}{\mbox{\it Outfamily}}
\newcommand{\C}{\ensuremath{C}}
\newcommand{\D}{\ensuremath{C}}

\begin{document}

\mainmatter  

\title{Prioritized Norms  in Formal Argumentation} 

\titlerunning{Prioritized Norms  in Formal Argumentation}

%
%
\author{Beishui Liao\inst{1,3}, Nir Oren\inst{2}, Leendert van der Torre\inst{3} and Serena Villata\inst{4} } 
%
\authorrunning{B. Liao N. Oren,  L. van der Torre and S. Villata} 

\institute{
Zhejiang University, China
\and
University of Aberdeen, United Kingdom
\and
University of Luxembourg, Luxembourg
\and
CNRS, I3S Laboratory, France
}

%
%

\toctitle{Lecture Notes in Computer Science}
\tocauthor{Authors' Instructions}
\maketitle

  
 \begin{abstract}
To resolve conflicts among norms, various nonmonotonic formalisms can be used to perform prioritized normative reasoning.   
Meanwhile, formal argumentation provides a way to represent nonmonotonic logics. 
In this paper, we  propose a representation of prioritized normative reasoning by argumentation. Using hierarchical abstract normative systems, we define three kinds of prioritized normative reasoning approaches, called \textit{Greedy}, \textit{Reduction}, and \textit{Optimization}. 
Then,  after formulating an argumentation theory for a hierarchical abstract normative system, we show that for a totally ordered hierarchical abstract normative system, Greedy and Reduction can be represented in argumentation by applying the weakest link and the last link principles respectively, and Optimization can be represented by introducing additional defeats capturing the idea that for each argument that contains a norm not belonging to the maximal obeyable set then this argument should be rejected. 

\end{abstract}



\section{Introduction}
Since the work of Alchourr\'{o}n and Makinson~\cite{alchourron} on hierarchies of regulations and their logic, in which a partial ordering on a code of laws or regulations is used to overcome logical imperfections in the code itself,  
reasoning with prioritized norms has been a central challenge in deontic logic~\cite{Hansen06,DBLP:conf/dagstuhl/BoellaTV08,parentvandertorre2013}. 

The goal of this paper is to study the open issue of reasoning with priorities over norms through the lens of argumentation theory~\cite{DBLP:journals/ai/Dung95}. 
More precisely, we focus on reasoning with the abstract normative system proposed by Tosatto \textit{et al.}~\cite{DBLP:conf/kr/TosattoBTV12}, which in turn is based on Makinson and van der Torre's approach to input/output logic~\cite{DBLP:journals/jphil/MakinsonT00}. In this system, an abstract norm is represented by an ordered pair $(a, x)$, where the body of the norm $a$ is thought of as an input, representing some kind of condition or situation, and the head of the norm $x$ is thought of as an output, representing what the norm tells us to be obligatory 
in that situation $a$. As a consequence, an abstract normative system is a directed graph $(L, N)$ together with a context $\D \subseteq L$ capturing base facts, where $L$ is a set of nodes, and $N\subseteq L\times L$ is the set of abstract norms. When the edge of an abstract normative system is associated with a number to indicate its priority over the other norms in the system, we obtain a hierarchical abstract normative system (HANS), which will be formally defined and studied in the remainder of this paper.  

Let us clarify how a hierarchical abstract normative system is defined by considering the  well known Order Puzzle \cite{Horty2007} example from the deontic logic literature, which revolves around three norms.









\begin{quote}
``Suppose that there is an agent, called Corporal O'Reilly, and that he is subject to the commands of three superior officers: a Captain, a Major, and a Colonel. The Captain, who does not like to be cold, issues a standing order that, during the winter, the heat should be turned on. The Major, who is concerned about energy conservation, issues an order that, during the winter, the window should not be opened. And the Colonel, who does not like to be too warm and does not care about energy conservation, issues an order that, whenever the heat is on, the window should be opened.'' 
\end{quote}
\begin{figure}[h!]
  \centering
\includegraphics[width=0.88\textwidth]{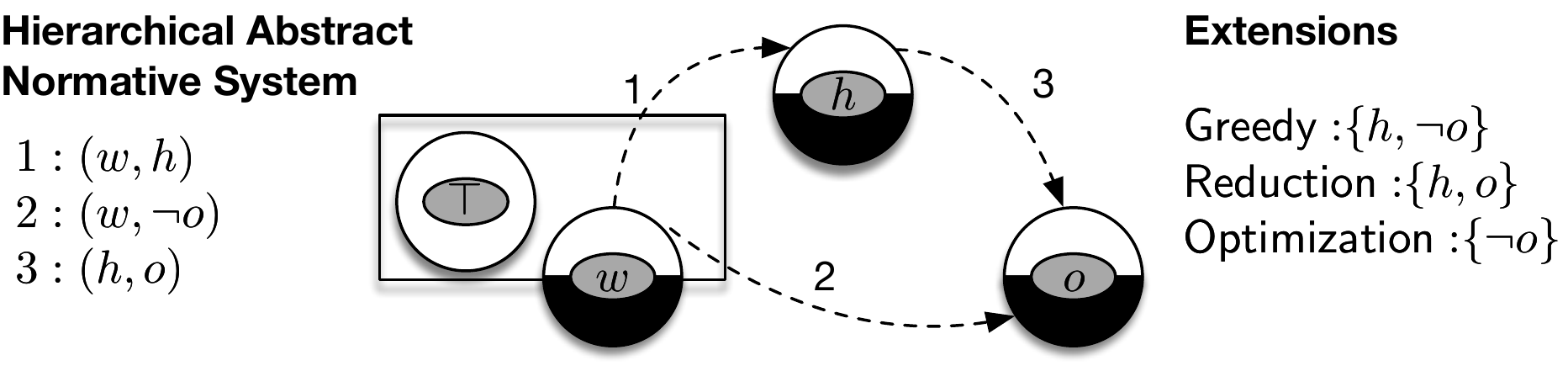}
    \caption{The Order puzzle example, represented using the graphical notation of Tosatto \textit{et al}.~ \cite{DBLP:conf/deon/TosattoBTV12} with edges annotated by norm strength.}
    \label{fig:ex-1}
\end{figure}

Let  $w$, $h$ and $o$ respectively denote the propositions that it is winter, the heat is turned on, and the window is open. There are three norms 
$(w, h)$, $(w, \neg o)$ and $(h, o)$. 

These three norms are  visualized in Figure~\ref{fig:ex-1}, extending the graphical notation described in Tosatto \textit{et al}. \cite{DBLP:conf/deon/TosattoBTV12} by associating edges with numbers  denoting priorities of norms. These priorities are obtained from the rank of the issuer, since Colonels outrank Majors, and Majors outrank  Captains. Within the figure, each circle denotes a proposition; the light part of the circle is the proposition itself, while the dark part denotes a negated proposition. Dashed lines represent the conditional obligations. 
Within Figure~\ref{fig:ex-1}, the line from the light part of $w$ to the dark part of $o$ denotes $(w, \neg o)$. The box on the left represents the context, in the example containing $\top$ and $w$.

In formalizing examples, one problem in applied logic is that the representation may be challenged. For example, it may be argued that the Colonel implies that if the window is closed, then the heating should be turned off. However, in normative systems, such pragmatic considerations are usually not part of the detachment procedure \cite{Risto2013}, with only explicitly given norms and commands being considered. Therefore, any such additional interpretations or other pragmatic concerns are out of the scope of this paper.

The central notion of inference in normative systems is called \emph{detachment}. For example, in the Order Puzzle, the fundamental question is whether we can detach $o$, $\neg o$, or both. In the example, the formulas which can be derived from a normative system are obligations. In general, permissions and institutional facts can also be detached from normative systems. 
A detachment procedure therefore defines the way deontic facts are derived from a normative system. Different detachment procedures have been defined and studied in deontic logic, as well as in other rule based systems. Moreover, even in hierarchical normative systems, not all conflicts may be resolved. In such a case, the detachment procedure may derive several so-called \emph{extensions}, each representing a set of obligations, permissions and institutional facts.

Abstract normative systems \cite{DBLP:conf/kr/TosattoBTV12} were introduced as a common core representation for normative systems, which are still expressive enough to define the main detachment procedures. In particular, analogous to the main input/output logics,  they have factual detachment built in, and have reasoning by cases, deontic detachment and identity as optional inference patterns \cite{DBLP:conf/icail/BoelleT03,Risto2013}. Such systems are called  `abstract', because negation is the only logical connective that is defined in the language. Furthermore, Tosatto \textit{et al}.~\cite{DBLP:conf/kr/TosattoBTV12} considered elements and anti-elements rather than literals and propositions. It is straightforward to define more connectives within such systems, and it is also possible to define structured normative systems where the abstract elements are instantiated with logical formulas, for example with formulas of a propositional or modal logic. The latter more interesting representation of logical structure is analogous to the use of abstract arguments in formal argumentation. An advantage of abstract normative systems over structured ones is that the central inference of detachment can be visualized by walking paths in the graph. In other words, inference is represented by graph reachability. For example, in the Order puzzle, node $o$ is reachable from the context, and thus it can be detached. Moreover, a conflict is represented by a node where both its light and the dark side are reachable from the context, as in the case of node $o$ in Figure 1. 

There are several optional inference patterns for abstract normative systems, because, as is well known, most principles of deontic logic have been criticized for some examples and applications. However, the absence of the same inference patterns is often criticized as well due to the lack of explanations and predictions of the detachment procedures resulting from these patterns. Therefore, current approaches to represent and reason with normative systems, such as input/output logic as well as abstract normative systems, do not restrict themselves to a single logic, but define a family of logics which can be embedded within them. Deciding which logic to use in a specific context depends on the requirements of the application. Similarly, with regards to permissions, there is an even larger diversity of deontic logics \cite{hansson2013} which adopt  different representations. For example, for  each input/output logic, various notions of permission have been defined, in terms of their relation to obligation.  We refer to the work of Parent and van der Torre \cite{parentvandertorre2013} for further explanation, discussion and motivation of this topic.


Now let us consider how various  detachment procedures might apply  norms  in hierarchical normative systems in different orders, and  result in different outcomes or extensions. Since detachment operates differently over permissions and obligations, and since the former introduces significant extra complexity, 
we only consider regulative norms and obligation detachment procedures in this paper.
We examine three approaches describing well known procedures defined in the literature 
\cite{Young2016,DBLP:conf/ijcai/Brewka89,hansen:makbook}.  These  procedures all share  one important property, namely that the context itself is not necessarily part of the procedure's output. It is precisely this feature which distinguishes input/output logics from traditional rule based languages like logic programming or default logic \cite{DBLP:journals/jancl/PrakkenS97,DBLP:conf/aaai/Brewka94}. Such traditional rule based languages where the input is part of the output are called throughput operators in input/output logic research. Therefore, the procedures we consider can naturally be captured using input/output logics. We refer to the three approaches as Greedy, Reduction and Optimization.

\textit{Greedy}: The context contains propositions that are known to hold.  This procedure always applies the norm with the highest priority that does not introduce inconsistency to an extension and the context. Here we say that a norm is applicable when its body is in the context or has been produced by other norms and added to the extension.  In the Order puzzle  example, we begin with  the context $\{w\}$, and $(w, \neg o)$ is first applied. Then $(w,h)$ is applied. Finally,  $(h,o)$ cannot be applied as this would result in a conflict, and so,  by using Greedy, we obtain the extension $\{h,\neg o \}$. 

\textit{Reduction}: in this approach, a candidate extension is guessed. All norms that are applicable according to this candidate extension are selected and transformed into unconditional or body-free norms. For example, a norm $(a,b)$ selected in this way is transformed to a norm $(\top,b)$. The modified hierarchical abstract normative system, with the transformed norms is then evaluated using  Greedy. The candidate extension is selected as an extension by  Reduction  if it is identified as an extension according to this application of Greedy.
In our example, we select a candidate extension  $\{h, o\}$,  obtaining a set of body-free norms $\{(\top, h), (\top, \neg o), (\top, o)\}$. The priorities assigned to these norms are carried through from the original hierarchical abstract normative system, and are therefore respectively 1, 2 and 3.  After applying Greedy, we get an  extension of Reduction: $\{h, o\}$. However, if we had selected the candidate extension $\{h, \neg o\}$, then this new extension would not appear in Greedy  as $(\top, \neg o)$ has a lower priority than $(\top, o)$, and the latter is therefore not an extension of Reduction.

\textit{Optimization}:  In terms of  Hansen's prioritized conditional imperatives, a set of maximally \emph{obeyable} (i.e., minimally violated) norms is selected by choosing norms in order of priority which are consistent with the context. Once these norms are selected,  Greedy is applied to identify the extension. In our example, the maximal set of obeyable norms is $\{(h, o), (w, \neg o)\}$.   Optimization therefore detaches the unique extension $\{\neg o\}$.  
We can also consider the example in terms of formal argumentation.  Given a hierarchical abstract normative system, we may construct an argumentation framework. The top part of Figure~\ref{fig:ex-2} illustrates the argumentation framework obtained for the Order puzzle which does not consider the priority relation between arguments.


An argumentation framework is a directed graph in which nodes denote arguments, and edges denote attacks between arguments. In the setting of a hierarchical abstract normative system, an argument is represented  as a path within the directed graph starting from a node in the context.  In this example, there are four arguments $A_0, A_1, A_2$ and $A_3$, represented as $[w], [(w, h)], [(w, h), (h,o)]$ and $[(w,\neg o)]$, respectively. Since the conclusions of $A_2$ and $A_3$ are inconsistent, $A_2$ attacks $A_3$ and vice-versa. To obtain correct conclusions, we must take priorities between arguments into account, transforming attacks into \emph{defeats}. This in turn requires \emph{lifting} the priorities given for the constituents of an argument to a priority for the argument itself, and two different principles have been commonly used for such a lifting. 
The \emph{last link} principle ranks an argument  based on the strength of its last inference, while the \emph{weakest link} principle ranks an argument based on the strength of its weakest inference. In the order example, if the {last link} principle is applied, then $[(w, h), (h,o)]$ defeats $[(w,\neg o)]$. The corresponding argumentation framework is illustrated in the middle portion of Figure \ref{fig:ex-2}. If the {weakest link} principle is used instead, then $[(w,\neg o)]$ defeats $[(w, h), (h,o)]$. The corresponding argumentation framework is illustrated at the bottom of Figure \ref{fig:ex-2}. As a result, the former principle allows us to conclude $\{h, o\}$, while the latter  concludes $\{h,\neg o\}$. In turn, the first result coincides with that obtained by Reduction, while the second is the same as that obtained by Greedy.

\begin{figure}
  \centering
  \fbox{
\includegraphics[width=0.55\textwidth]{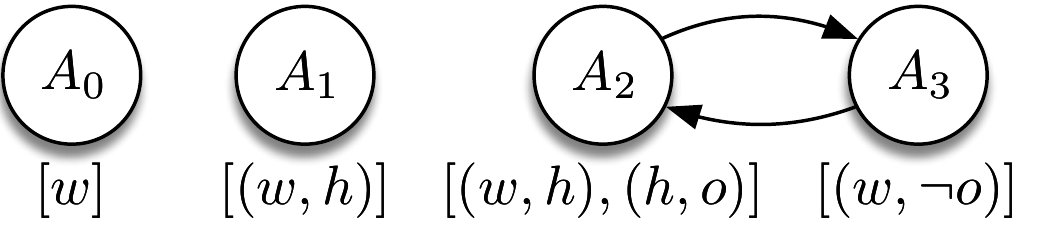}}
\fbox{
\includegraphics[width=0.55\textwidth]{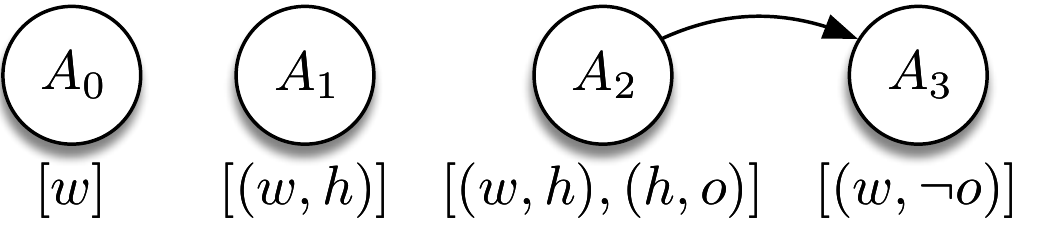}}
\fbox{
\includegraphics[width=0.55\textwidth]{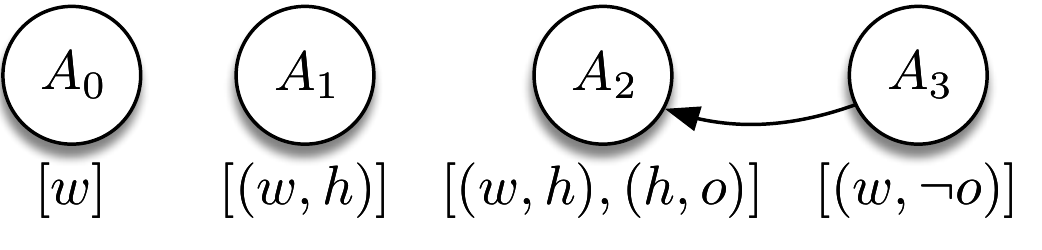}}

    \caption{Top: the argumentation framework obtained from the order puzzle hierarchical normative system with the four arguments and attacks between them visualized as directed arrows. Middle: the argumentation framework obtained when last link is applied. Bottom: the argumentation framework obtained when weakest link is applied.}
    \label{fig:ex-2}
\end{figure}



Inspired by the example  above, we wish to investigate the links between the three detachment procedures for prioritized normative reasoning and argumentation theory. More specifically, our main research question is as follows.

\begin{quote}
\emph{
How can the three detachment procedures (Reduction, Greedy, and Optimization) proposed in the context  of abstract normative reasoning be represented in formal argumentation?}
\end{quote}

To answer this research question, we propose a formal framework to connect hierarchical normative reasoning with argumentation theory. 
More precisely, our framework represents the above-mentioned detachment procedures by lifting priorities from norms to arguments, 
with the underlying goal of making as few commitments as possible to specific argumentation systems. For this reason, we build on a structured argumentation framework which admits undercuts and rebuts between arguments, and allows for priorities between norms making up arguments. We show that variants of approaches to lifting priorities from rules to arguments allow us to capture both Greedy and Reduction, while the introduction of additional defeats  allows us to obtain 
Optimization. 

Some preliminary results found in this paper were originally presented in DEON 2016 \cite{bnls2016}. We keep, but extend, the representation results for Greedy and Reduction (previously called the ``Brewka-Eiter construction''). To capture  Optimization (previously called ``Hansen's construction'')  the DEON 2016 paper adds new arguments through the introduction of additional  permissive norms. In the current paper, this is replaced with an approach which introduces additional defeats. The reasons for this replacement are two-fold. First, the preliminary approaches for Optimization from the earlier work \cite{bnls2016} cannot fully capture the idea that  each argument containing a norm not belonging to the maximal obeyable set should be rejected. Second, permissive norms were used as a means to support the representation, which making the approach more complicated. 



The layout of the remainder of the paper is as follows.  Section 2 formalizes the above-mentioned three detachment procedures of hierarchical normative reasoning (i.e., Greedy,  Reduction, and Optimization). In Section 3, we introduce an argumentation theory for a hierarchical abstract normative system. Sections 4, 5 and 6  show how Greedy, Reduction and Optimization can be represented in argumentation. Finally, in Section 7 we discuss open problems and compare the proposed approach with related work, and in Section 8 we point out possible directions for future work.

\section{Hierarchical abstract normative systems} \label{sect-pns}
In this section, we formally introduce the notion of hierarchical abstract normative systems and the three detachment procedures to compute their normative conclusions. A hierarchical abstract normative system captures the context of a system and the norms in force in such a system. 
There is an element in the universe called $\top$, contained in every context. In this paper, we consider only a finite universe. A hierarchical abstract normative system also encodes a ranking function over the norms to allow for the resolution of conflicts.

Based on the notion of \textit{abstract normative system} defined by Tosatto \textit{et al}.~\cite{DBLP:conf/kr/TosattoBTV12}, a hierarchical abstract normative system can be defined as follows. 

\begin{definition}[Hierarchical abstract normative system] \label{def-pans}
A hierarchical abstract normative system is a tuple $\mathcal{H} = \langle L, N, \C, r\rangle$, where
\begin{itemize}
\item $L = \mathrm{E} \cup \{\neg e \mid e\in \mathrm{E}\} \cup \{\top\}$ is the universe, a set of literals based on some finite set $\mathrm{E}$ of atomic elements;
\item $N\subseteq L\times L$ is a  finite set of regulative  norms;
\item $\C \subseteq L$ is a subset of the universe, called a context, such that $\top \in \C$ and for all $e$ in $\mathrm{E}$, $\{e, \neg e\} \not\subseteq \C$;
\item $r: N \rightarrow \N$ is a function from  norms to  natural numbers.
\end{itemize}
\end{definition}

Regulative (ordinary) norms are of the kind ``if you turn on the heat, then you should open the window''. 
These norms  are \textit{conditional norms}, requiring some condition to  hold (e.g., turning on the heat) before their conclusions can be drawn.

We write $(a,x)$ for a regulative norm, 
where $a, x\in L$ are the antecedent and conclusion of the norm, respectively. 
Given $(a,x)$, we use $r(a,x)$ to denote $r((a,x))$. Let $u, v\in N$ be two norms, we say that $v$ is at least as preferred as $u$ (denoted $u\le v$) if and only if $r(u)$ is not larger than $r(v)$ (denoted $r(u)\le r(v)$), where $r(u)$ is also called the rank of $u$. We write $u< v$ or $v > u$ if and only if $u\le v$ and $v\not\le u$. Given a norm $u = (a, x)$, we write $\mbox{\it ant}(u)$ for $a$ to represent the antecedent of the norm, and $\mbox{\it cons}(u)$ for $x$ to represent the consequent of the norm. Given a set of norms $S\subseteq N$, we use $\mbox{\it cons}(S)$ to denote $\{\mbox{\it cons}(u) \mid u\in S\}$. We say that a hierarchical abstract normative system is totally ordered if and only if the ordering $\le$ over $N$ is antisymmetric, transitive and total. Due to the finiteness of universe, the set of norms is finite. Note that given this assumption, the notion of total ordering here is identical to that of the full prioritization in Brewka and Eiter's~\cite{Brewka} and Hansen's~\cite{DBLP:journals/aamas/Hansen08} work, and of the linearized ordering of Young \textit{et al}. \cite{Young2016}.   
For $a\in L$, we write $\overline{a}=\neg a$ if and only if $a\in \mathrm{E}$, and $\overline{a}=e$ for $e\in \mathrm{E}$ if and only if $a=\neg e$. For a set $S \subseteq L$, we say 
 that $S$ is consistent if and only if there exist no $e_1, e_2\in S$ such that $e_1 = \overline{e_2}$. 
To exemplify the notions of a hierarchical abstract normative system, consider the following example. 

\begin{example}[Order puzzle]\label{pans-1}
In terms of Definition \ref{def-pans}, the set of norms and priorities that are visualized in Figure \ref{fig:ex-1} can be formally represented as a hierarchical abstract normative system $\mathcal{H} = \langle L, N$, $\C$, $r\rangle$, where 
	 $L = \{w, h, o, \neg w, \neg h, \neg o, \top\}$, $N  = \{(w,h), (h, o), (w, \neg o)\} $,  $ \C = \{w, \top\}$, $r(w,h) =1$,   
	 $r(h,o) =3$,    
	 $r(w, \neg o) =2$.  
\end{example}

In the hierarchical abstract normative system setting, the three detachment procedures for prioritized normative reasoning can be defined as follows.

Firstly,  Greedy detachment for a hierarchical abstract normative system always applies the norm with the highest priority among those which can be applied, if this does not bring inconsistency to the extension and the context. To formally define the notion of Greedy, we first introduce the following notions of paths and consistent paths. 


\begin{definition}[Path, consistent path]
Let $\mathcal{H} = \langle L$, $N,\C, r\rangle$ be a hierarchical abstract normative system. 
\begin{itemize}
\item A path in $\mathcal{H}$ from $x_1$ to $x_n$ is a sequence of norms $(x_1, x_2)$, $(x_2, x_3)$, $\dots,  (x_{n-1},x_n)$ such that $\{(x_1, x_2), (x_2, x_3), \dots, (x_{n-1}$, $x_n)\}\subseteq N$, $n\geq 2$, and all norms of the sequence are distinct. 
\item A path in $\mathcal{H}$ from $x_1$ to $x_n$ with respect to $R\subseteq N$  is a sequence of norms $(x_1, x_2)$, $(x_2, x_3), \dots,  (x_{n-1},x_n)$ such that $\{(x_1, x_2), (x_2, x_3), \dots, (x_{n-1}$, $x_n)\}\subseteq R$, $n\geq 2$, and all norms of the sequence are distinct. 
\item A consistent path in $\mathcal{H}$ from $x_1$ to $x_n$ (with respect to $R$) is a path $(x_1, x_2)$, $(x_2, x_3), \dots,  (x_{n-1},x_n)$ such that $\{x_1, x_2, \dots, x_n\}$ is consistent.
\end{itemize}
\end{definition}

Then, the unique extension of a totally ordered hierarchical abstract normative system by Greedy can be defined as follows by first selecting a set of applicable norms with the highest priority in each step, and then collecting the final vertex of each path with respect to the set of selected norms. 

\begin{definition}[Greedy] \label{def-greedy}
Let $\mathcal{H} = \langle L$, $N, \C, r\rangle$ be a totally ordered hierarchical abstract normative system.  For all $R\subseteq N$, let $R(\C) = \{x\mid$ there is a path in $\mathcal{H} $ from an element in $\C$ to $x$ with respect to $R\}$, and $\mbox{\it Appl}(N, \C, R) := \{(a,x)\in N\setminus R \mid a\in \C \cup \mbox{\it cons}(R)$, $\{x, \overline{x}\}$$\not\subseteq \C \cup \mbox{\it cons}(R)\}$. 
The extension of  $\mathcal{H}$ by Greedy, written as $\mbox{\it Greedy}(\mathcal{H})$, is the set $R(C)$ such that $R =  \cup^\infty_{i=0} R_i$ is built inductively as follows.
\begin{eqnarray*}
R_0 & = &  \emptyset \\
R_{i+1} &=&  R_i\cup \mathrm{max}(N,\C, R_i,r)
\end{eqnarray*}
where $\mathrm{max}(N,\C, R_i,r) = \{u\in \mbox{\it Appl}(N,  \C, R_i) \mid \forall v\in  \mbox{\it Appl}(N,  \C, R_i): r(u) \geq r(v)\}$. 

%
\end{definition}

In terms of Definition \ref{def-greedy}, the unique extension of the Order puzzle can be constructed as follows. 

\begin{example}[Extension by Greedy]\label{pans-1b}
Given $\mathcal{H}$ in Example~\ref{pans-1}, by Greedy, it holds that $R_0=\emptyset$,  
 $R_1=\{(w,\neg o)\}$, 
 $R_2 = \{(w, \neg o), (w,  h)\}$, 
 $R = \{(w, \neg o), (w,  h)\}$.
So, $\mbox{\it Greedy}(\mathcal{H})=\{h, \neg o\}$.
\end{example}

Reduction is defined by first using a candidate extension to get a modified hierarchical abstract normative system, and then applying Greedy to it. If the candidate extension is an extension according to this application of Greedy, then it is an extension of the original hierarchical abstract normative system by Reduction. Formally, we have the following definition. 



\begin{definition}[Reduction]\label{def:brewka}
Given a totally ordered hierarchical abstract normative system $\mathcal{H}=\langle L, N$,  $\C$, $r\rangle$, and a set $X$, 
let $\mathcal{H}^X$=$\langle L, N'$,  $\C$, $r'\rangle$, where
%
$N' = \{(\top, x_2) \mid (x_1, x_2)\in N, x_1\in \C \cup X\}$ is the set of ordinary norms, 
and $r'(\top, x_2) = \max(r(x_1, x_2)\mid   (x_1, x_2)\in N, x_1\in \C \cup X)$ for all $(x_1, x_2)\in N$ are priorities over norms.
An extension of $\mathcal{H}$ by Reduction is a set $U$ such that 
$U$ is $\mbox{\it Greedy}(\mathcal{H}^U)$.  The set of extensions of $\mathcal{H}$ by Reduction  is denoted as $\mbox{\it Reduction}(\mathcal{H})$.
\end{definition}

According to Definition \ref{def:brewka}, it can be the case that different norms have the same consequent. To avoid the duplication of multiple body-free norms, only a single norm with the highest priority is used. 

To exemplify the notion of Reduction, consider again the Order puzzle in Example~\ref{pans-1}.

\begin{example}[Extensions by Reduction]\label{pans-1c}
By using Reduction, given $X = \{ h, o\}$,  we have $\mathcal{H}^X = \langle L, N^\prime,  \C, r^\prime\rangle$, \mbox{where}
 $N^\prime = \{(\top,h), (\top,o), (\top,\neg o)\}$, 
 $r^\prime(\top,h) = 1$, $r^\prime(\top, o) = 3$ and $r^\prime(\top,\neg o) = 2$. 
 Since $X\in\mbox{\it Greedy}(\mathcal{H}^X)$, and no other set can be an extension, we have that $\mbox{\it Reduction}(\mathcal{H}) = \{ \{h, o\}\}$.


\end{example}

Note that a totally ordered hierarchical abstract normative system might have more than one extension, as illustrated by the following example.

\begin{example}[Multiple extensions by Reduction]
Given the hierarchical abstract normative system in Figure \ref{fig:poke2}, assume that we have a context $C=\{a\}$. We then consider $X_1=\{b,c\}$ and $X_2=\{\neg b\}$. In the first case, we have $\mathcal{H}^{X_1}=\langle L,N',C,r' \rangle$ where
$N'=\{(\top,b),(\top,\neg b),(\top,c)\}$ and $r'(\top,b)=4$, $r'(\top,c)=3$, $r'(\top,\neg b)=2$. Here, $\mathit{Greedy}(\mathcal{H}^{X_1})=\{b,c\}$,  i.e., $X_1$. 

In the second case, we obtain $\mathcal{H}^{X_2}=\langle L,N',C,r' \rangle$ where
$N'=\{(\top,b),(\top,\neg b)\}$ and $r'(\top,b)=1$, $r'(\top,\neg b)=2$. Now, $\mathit{Greedy}(\mathcal{H}^{X_2})=\{\neg b\}$, i.e., $X_2$. In this case, we therefore obtain two extensions using Reduction. 

\begin{figure}[h!]
\centering
\includegraphics[width=0.7\textwidth]{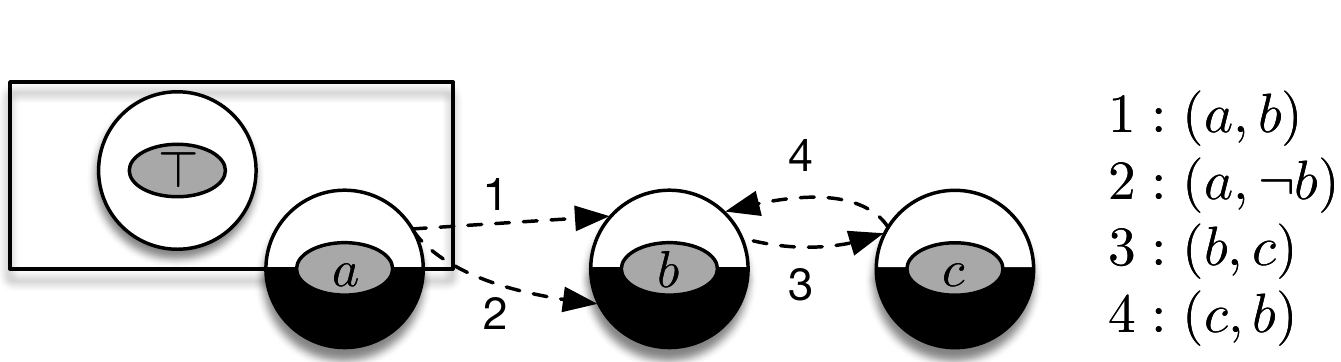}
\caption{\label{fig:poke2} The hierarchical abstract normative system of Example \ref{pans-1c} containing the two  Reduction extensions $\{b,c\}$ and $\{\neg b\}$.}
\end{figure}
\end{example}




Finally, Optimization can be defined by first choosing a set of maximally obeyable norms, and then applying Greedy to it.

\begin{definition}[Optimization] \label{minviolation}
Given a totally ordered hierarchical abstract normative system $\mathcal{H}=\langle L, N$,  $\C$, $r\rangle$,
let $T = (u_1, u_2, \dots, u_n)$ be the linear order on $N$ such that  $u_1> u_2>\dots>u_n$. 
We define a sequence of sets of norms $R_0\dots R_n$ as follows:
\begin{eqnarray*}
R_0 &= & \emptyset \\
R_{i+1} &=&
\left\{ \begin{array}{l}
R_i \cup \{u_i\}, \mbox{ if } \C \cup R(\C)\mbox{ is consistent  where } R=R_i \cup \{u_i\}  \\
R_i, \mbox{ else} \end{array} \right.
\end{eqnarray*}
The unique extension of $\mathcal{H}$ by Optimization is $R_n(C)$ and denoted as $\mbox{\it Optimization}(\mathcal{H})$.
\end{definition}

According to Definition \ref{minviolation}, the unique extension of the Order puzzle can be constructed as follows. 

\begin{example}[Extension by Optimization]\label{ex:extensions}
 Regarding $\mathcal{H}$ in Example~\ref{pans-1}, by 
 Optimization, let $u_1 =(h, o)$,  $u_2 = (w$, $ \neg o)$, and $u_3 = (w,h)$, and $T=(u_1,u_2, u_3)$. Then,  it holds that $R_0= \emptyset$, $R_1= \{u_1\}$, 
  $R_2= \{u_1, u_2\}$, and 
 $R = R_3 = R_2= \{u_1, u_2\}$. 
 So, we obtain that $ \mbox{\it Optimization}(\mathcal{H}) = \{\neg o\}$.   

\end{example}

\section{Argumentation theory for a hierarchical abstract normative system}\label{sec:argPANS}

In this section, we introduce an argumentation theory on prioritized norms. 
Given a hierarchical abstract normative system, we first define arguments and defeats between them, then compute extensions of arguments in terms of Dung's theory~\cite{DBLP:journals/ai/Dung95}, and from these, obtain conclusions.

\subsection{Arguments}
In a hierarchical abstract normative system, an argument is an acyclic path in the graph starting in an element of the context. We assume minimal arguments---no norm can be applied twice in an argument and no redundant norm is included in an argument. 
We use $concl(\alpha)$ to denote the conclusion of an argument $\alpha$, and $concl(E) = \{concl(\alpha) \mid \alpha\in E\}$ for the conclusions of a set of arguments $E$.

\begin{definition}[Arguments and sub-arguments] \label{argument}
Let $\mathcal{H} = \langle L, N,  \C, r\rangle$ be a hierarchical abstract normative system.
\begin{description}
\item[A context argument] in $\mathcal{H}$ is an element $a\in \C$, and its conclusion is $concl(a)=a$.
\item[An ordinary argument] is a consistent path $\alpha$ in $\mathcal{H}$ from $x\in \C$ to some $y\in L$. 
Moreover, we have that $concl(\alpha)= y$.
\item[The sub-arguments] of argument $[u_1,\ldots, u_n]$ are, for $1\leq i\leq  n$, $[u_1,\ldots, u_i]$. Note that context arguments do not have sub-arguments. 
\end{description}
\end{definition}

The set of all arguments constructed from $\mathcal{H}$ is denoted as $\mbox{\it Arg}(\mathcal{H})$. For readability, $[(a_1,a_2)$, $\ldots,(a_{n-1},a_n)]$ may be written as $(a_1,a_2, \ldots, a_{n-1}, a_n)$. The set of sub-arguments of an argument $\alpha$ is denoted as $\mbox{\it sub}(\alpha)$. When $\alpha$ is a sub-argument of $\beta$, we say that $\beta$ is a super-argument of $\alpha$. The set of super-arguments of $\alpha$ is denoted as $\mbox{\it sup}(\alpha)$. Furthermore, the set of proper sub-arguments of $\alpha$ is defined as $\mbox{\it psub}(\alpha) = \mbox{\it sub}(\alpha)\setminus \{\alpha\}$. For an ordinary argument $\alpha = [u_1,\ldots, u_n]$, we call $u_n$ the \emph{top norm} of $\alpha$. Finally,  in each ordinary argument, the norm with lowest priority is called the \emph{weakest norm} of the argument. Formally, we have the following definition for this latter concept.

\begin{definition}[The weakest norm of an argument]
Let $\mathcal{H}$ be a totally ordered hierarchical abstract normative system.  For an ordinary argument $\alpha = [u_1,\ldots, u_n]$ in $\mathcal{H}$, if $r(u_i)=\min(r(u_j)\mid 1\le j\le n)$, then $u_i$ is called the weakest norm of $\alpha$. 
\end{definition}

\subsection{Defeat relation between arguments} 
We follow the tradition in much of preference-based argumentation~\cite{DBLP:conf/sum/AmgoudV10,DBLP:journals/ai/ModgilP13}, where  \emph{attack} captures a relation among arguments which ignores preferences, and {\em defeat} is a preference-aware relation on which the semantics is based. 
To define the defeat relation among prioritized arguments, we assume that {\em only} the priorities of the norms are used to compare arguments. In other words, we assume a lifting of the ordering on norms to a binary relation on sequences of norms (i.e., arguments), written as $\alpha\succeq \beta$, where $\alpha$ and $\beta$ are two arguments, indicating that $\alpha$ is at least as preferred as $\beta$. 




There is no common agreement about the best way to lift $\geq$ to $\succeq$. In argumentation, there are at least two ways to introduce weights. As an argumentation framework consists of a set of arguments and an attack relation between them, we can either assign weights to arguments, or we can assign weights to attacks. Traditionally, weights are assigned to arguments. Two common approaches to give the strength of an argument are the \textit{weakest link} and the \textit{last link} principles, combined with the elitist and democratic ordering~\cite{DBLP:journals/ai/ModgilP13}. For example, in the weakest link principle the weight of the argument is the weight of the weakest rule used in the argument.  However, Young \textit{et al.}~\cite{Young2016} showed that elitist weakest link cannot be used to calculate $\succeq$ for Greedy, and proposes a \emph{disjoint elitist order} which ignores shared rules. It is worth noticing that the strength of an argument may depend on the argument it is attacking, as identified by Young \textit{et al}.~\cite{Young2016}. Based on this idea, we define the orderings between arguments by assigning a strength to the attacks between the arguments, to reflect the priority of the norms used in the arguments, following the same insights of the weakest link and last link principles (denoted as $\succeq_w$ and $\succeq_l$ respectively). Defining the weakest link ordering in the same way as Young \textit{et al}.~\cite{Young2016}, we have the following definition.



\begin{definition}[Weakest link and last link] \label{def-w-l}
Let $\mathcal{H}=\langle L, N$,  $\C$, $r\rangle$ be a hierarchical abstract normative system, and $\alpha = [u_1, \dots, u_n]$ and $\beta = [v_1,\dots, v_m]$ be two arguments in $\mbox{\it Arg}(\mathcal{H})$. Let $\Phi_1 = \{u_1, \dots, u_n\}$ and $\Phi_2 = \{v_1, \dots, v_m\}$.
By the weakest link principle,  $\alpha\succeq_{w} \beta$ iff  $\exists v\in \Phi_2\setminus \Phi_1$ such that $\forall u\in \Phi_1\setminus \Phi_2, v \le u$.
By the last link principle, $\alpha\succeq_{l} \beta$ iff $u_n\geq v_m$.
\end{definition}


When the context is clear, we write $\succeq$ for $\succeq_{w}$, or $\succeq_l$. We write $\alpha\succ \beta$ for $\alpha \succeq \beta$ without $\beta\succeq \alpha$. The following proposition shows the transitivity of the two relations $\succeq_{w}$ and $\succeq_l$.

\begin{proposition}[Transitivity] \label{prop-transitive}
It holds that the relations $\succeq_{w}$ and $\succeq_l$ are transitive. 
\end{proposition}

\begin{proof}
Let $\alpha = [u_1, \dots, u_n]$, $\beta = [v_1,\dots, v_m]$ and $\gamma = [w_1,\dots, w_k]$ be three arguments in $\mbox{\it Arg}(\mathcal{H})$. For the case of the weakest link, let $\Phi_1 = \{u_1, \dots, u_n\}$, $\Phi_2 = \{v_1,\dots, v_m\}$, and $\Phi_3 = \{w_1,\dots, w_k\}$, $n,m,k\ge 1$. Let  
$x_{12}\in \Phi_1\cap \Phi_2$, $x_{13}\in \Phi_1\cap\Phi_3$, and $x_{23}\in \Phi_2\cup \Phi_3$, $x_{1}\in \Phi_1\setminus (\Phi_2\cup \Phi_3)$, $x_{2}\in \Phi_2\setminus (\Phi_1\cup \Phi_3)$, and $x_{3}\in \Phi_3\setminus (\Phi_1\cup \Phi_2)$. Assume that $\alpha\succeq_{w}\beta$ and $\beta\succeq_{w}\gamma$. There are only the following four possible cases.

Case 1: There exists $x_{23}\in \Phi_2$, for all $x_{13}, x_1\in \Phi_1$,
 $x_{23} \le x_{13}$ and $x_{23} \le x_{1}$; and since $x_{13}\in \Phi_3$, assume that for all $x_{12}, x_2\in \Phi_2$, $x_{13} \le x_{12}$, $x_{13} \le x_{2}$. It follows that $x_{23} \le x_{12}$ and $x_{23} \le x_{1}$. Since $x_{23} \in \Phi_3$, it means that there exists $x_{23} \in \Phi_3$ such that $x_{23} \le x_{12}$ and $x_{23} \le x_{1}$ where $x_{12}, x_{1}\in \Phi_1$. Hence, $\alpha\succeq_{w}\gamma$.

Case 2: There exist $x_{23}\in \Phi_2$ and $x_{3}\in \Phi_3$,  for all $x_{13}, x_1\in \Phi_1$, and $x_{12}, x_2\in \Phi_2$: $x_{23} \le x_{13}$, $x_{23} \le x_{1}$, $x_{3} \le x_{12}$, $x_{3} \le x_{2}$. In this case, there are in turn only the following two possible sub-cases: either $x_{23} \le x_3$ or $x_{23} > x_3$. If $x_{23} \le x_3$, since $x_{3} \le x_{12}$, it holds that $x_{23} \le x_{12}$. Since $x_{23} \le x_{1}$ and $x_{23} \le x_{12}$, it holds that $\alpha\succeq_{w}\gamma$. Second, if $x_{23} > x_3$, since $x_{23} \le x_{1}$, $x_{3} \le x_{1}$. Since $x_{3} \le x_{12}$ and $x_{3} \le x_{1}$, $\alpha\succeq_{w}\gamma$. 

Case 3: There exists $x_{2}\in \Phi_2$, for all $x_{13}, x_1\in \Phi_1$,
 $x_{2} \le x_{13}$ and $x_{2} \le x_{1}$; and since $x_{13}\in \Phi_3$, assume that for all $x_{12}, x_2\in \Phi_2$, $x_{13} \le x_{12}$, $x_{13} \le x_{2}$. In this case, there are in turn only the following two possible sub-cases: $x_{23} \le x_{13}$ or $x_{3} \le x_{13}$, or  $x_{23} > x_{13}$ and $x_{3} > x_{13}$. If $x_{23} \le x_{13}$ and $x_{3} \le x_{13}$, it holds that $\alpha\succeq_{w}\gamma$. If $x_{23} > x_{13}$ and $x_{3} > x_{13}$, it holds that $x_{13} < x_{23}$ and $x_{13} \le x_{2}$, and therefore $\beta\succeq_{w}\alpha$. Contradiction. 

Case 4: There exist $x_{2}\in \Phi_2$ and $x_{3}\in \Phi_3$,  for all $x_{13}, x_1\in \Phi_1$, and $x_{12}, x_2\in \Phi_2$: $x_{2} \le x_{13}$, $x_{2} \le x_{1}$, $x_{3} \le x_{12}$, $x_{3} \le x_{2}$. In this case, since $x_{3} \le x_{12}$ and $x_{3} \le x_1$, it holds that $\alpha\succeq_{w}\gamma$.

For last link, if $\alpha\succeq_{l}\beta$ and $\beta\succeq_{l}\gamma$, then $u_n\ge v_m$ and $v_m\ge w_k$. It follows that $u_n\ge w_k$, and therefore $\alpha\succeq_{l}\gamma$. So, it holds that  $\succeq_l$ is transitive.  
\end{proof}

Given a way to lift the ordering on norms to an ordering on arguments, the notion of defeat can be defined as follows. 

\begin{definition}[Defeat among arguments] \label{defeat}
Let $\mathcal{H} = \langle L,N$, $\C,r\rangle$ be a hierarchical abstract normative system. For all $\alpha, \beta\in \mbox{\it Arg}(\mathcal{H})$,
\begin{description}
\item[$\alpha$ attacks $\beta$] iff $\beta$ has a sub-argument $\beta'$ such that 
\begin{enumerate}
\item
$concl(\alpha)=\overline{concl(\beta')}$
\end{enumerate}
\item[$\alpha$ defeats $\beta$] iff $\beta$ has a sub-argument $\beta'$ such that 
\begin{enumerate}
\item
$concl(\alpha)=\overline{concl(\beta')}$ and 
\item 
$\alpha$ is a context argument; or $\alpha$ is an ordinary argument and $\alpha\succeq \beta$. 
\end{enumerate}
\end{description}
\end{definition}


The set of defeats between the arguments in $\mbox{\it Arg}(\mathcal{H})$ based on a preference ordering $\succeq$ is denoted as $\mbox{\it Def}(\mathcal{H},\succeq)$.

In what follows, an argument $\alpha = [u_1, \dots, u_n]$ with ranking on norms is denoted as $u_1\dots u_n: r(\alpha)$, where $r(\alpha) = (r(u_1), \dots, r(u_n))$.

The notions of arguments and defeat relations between arguments by the weakest link and the last link principles respectively can be illustrated by the following example. 

\begin{example}[Order puzzle continued] \label{argnattck}
Consider the hierarchical abstract normative system in Example~\ref{pans-1}. We have the following arguments (visually presented in the top of  Figure~\ref{fig:ex-2}):
\begin{description}
\item[]$A_0$ : $w$ \hfill (context argument)
\item[]$A_1$ : $(w,h):(1)$ \hfill (ordinary argument)
\item[]$A_2$ : $(w,h)(h, o):(1,3)$ \hfill (ordinary argument)
\item[]$A_3$ : $(w,\neg o):(2)$ \hfill (ordinary argument)
\end{description}
We have that $A_2$ attacks $A_3$ and vice versa, and there are no other attacks among the arguments.
Moreover, $A_2$ defeats $A_3$ by the last link principle (Figure \ref{fig:ex-2}, middle), and $A_3$ defeats $A_2$ by the weakest link principle (Figure \ref{fig:ex-2}, bottom).
\end{example}


\subsection{Argument extensions and conclusion extensions} Given a set of arguments $\mathcal{A} = \mbox{\it Arg}(\mathcal{H})$ and a set of defeats $\mathcal{R} =\mbox{\it Def}(\mathcal{H},\succeq)$, we get an argumentation framework (AF) $\mathcal{F} = (\mathcal{A}, \mathcal{R})$. 

Following the notions of abstract argumentation by Dung \cite{DBLP:journals/ai/Dung95}, we say that a set $B\subseteq \mathcal{A}$ is \emph{admissible}, if and only if it is \textit{conflict-free} and it can defend each argument within the set.  A set $B\subseteq \mathcal{A}$ is \emph{conflict-free} if and only if there exist no arguments $\alpha$ and $\beta$ in $B$ such that $(\alpha,\beta)\in \mathcal{R}$. Argument $\alpha\in \mathcal{A}$ is \textit{defended} by a set $B\subseteq \mathcal{A}$ (in such a situation $\alpha$ can also be said to be  \emph{acceptable} with respect to $B$) if and only if for all $\beta\in \mathcal{A}$, if $(\beta,\alpha)\in \mathcal{R}$, then there exists $\gamma\in B$ such that $(\gamma,\beta)\in \mathcal{R}$. Based on the notion of admissible sets, some other extensions can be defined. Formally, we have the following.

\begin{definition}[Conflict-freeness, defense and extensions] \label{Def-AF-conflict free}
{Let $\mathcal{F} = (\mathcal{A}, \mathcal{R})$ be an argumentation framework, and $B\subseteq \mathcal{A}$ a set of arguments.}
\begin{itemize}
 \item {$B$ is \emph{conflict-free} if and only if $\nexists \alpha, \beta\in B$, such that  $(\alpha,\beta)\in \mathcal{R}$.}
  \item {An argument $\alpha\in \mathcal{A}$ is defended by $B$ (equivalently  $\alpha$ is \emph{acceptable} with respect to $B$), if and only if $\forall(\beta,\alpha)\in \mathcal{R}$, $\exists\gamma\in B$, such that $(\gamma,\beta)\in \mathcal{R}$.}
  \item {$B$ is \emph{admissible} if and only if $B$ is conflict-free, and each argument in $B$ is defended by $B$.}
  \item {$B$ is a \emph{complete} extension if and only if $B$ is admissible and each argument in $\mathcal{A}$ that is defended by $B$ is in $B$.}
 \item {$B$ is a \emph{preferred} extension if and only if $B$ is a maximal (with respect to set-inclusion) complete extension.}
  \item {$B$ is a \emph{grounded} extension if and only if $B$ is the minimal (with respect to set-inclusion) complete extension.}
  \item {$B$ is a \emph{stable} extension if and only if $B$ is conflict-free, and $\forall \alpha\in \mathcal{A}\setminus B$, $\exists \beta\in B$ such that $(\beta,\alpha)\in \mathcal{R}$.}
\end{itemize}
\end{definition}

A semantics describes the set of extensions one wishes to obtain. 
We use $\mbox{\it sem} \in \{\mbox{\it cmp}, \mbox{\it prf}, \mbox{\it grd}, \mbox{\it stb}\}$ to denote the complete, preferred, grounded, and stable semantics, respectively. A set of argument extensions of $\mathcal{F} = (\mathcal{A}, \mathcal{R})$ is denoted as $\mbox{\it sem}(\mathcal{F})$. 
 We write $\Outfamily$ for the set of conclusions from the extensions of the argumentation theory ~\cite{normaspic}.

\begin{definition}[Conclusion extensions]
Given a hierarchical abstract normative system $\mathcal{H} = \langle L,N,\C,r\rangle$, let $\mathcal{F}  = (\mbox{\it Arg}(\mathcal{H}), \mbox{\it Def}(\mathcal{H},\succeq))$ be the AF constructed from $\mathcal{H}$.
The conclusion extensions, written as $\Outfamily(\mathcal{F},$ $\mbox{\it sem})$, are the conclusions of the ordinary arguments within argument extensions.
$$\mathit{Outfamily}(\mathcal{F},\mathit{sem})=\{\{concl(\alpha)\mid \alpha \in S, \alpha \mbox{ is an ordinary argument} \}  \mid S\in \mbox{\it sem}(\mathcal{F})\}$$
\end{definition}



The following example shows that for a hierarchical abstract normative system, when adopting different principles for lifting the priorities over norms to those over arguments, the resulting argumentation frameworks as well as their conclusion extensions may be different.

\begin{example}[Order puzzle in argumentation] \label{Ptc}
According to Example~\ref{argnattck}, let $\mathcal{A} = \{A_0$, $\dots, A_3\}$. We have $\mathcal{F}_1 = (\mathcal{A}$, $\{(A_2$, $A_3)\})$ where $A_2 \succeq_l A_3$,  and $\mathcal{F}_2 = (\mathcal{A}$, $\{(A_3, A_2)\})$ where $A_3 \succeq_w A_2$. For all $\mbox{\it sem} \in \{\mbox{\it cmp}, \mbox{\it prf}, \mbox{\it grd}, \mbox{\it stb}\}$, $\Outfamily(\mathcal{F}_1,$ $\mbox{\it sem}) = \{\{h,o\}\}$, and $\Outfamily(\mathcal{F}_2,$ $\mbox{\it sem}) = \{\{h, \neg o\}\}$.
%
\end{example}

We now turn our attention to the properties of the argumentation theory for a hierarchical abstract normative system. 

First, according to Definition \ref{defeat}, we have the following proposition, capturing the relation between the attack/defeat on an argument and on its super-arguments.

\begin{proposition}[Super-argument attack and defeat] \label{prop-defeat-sup}
Let $\mathcal{F} = (\mathcal{A}, \mathcal{R})$ be an AF constructed from a hierarchical abstract normative system. For all $\alpha, \beta\in \mathcal{A}$, if $\alpha$ attacks $\beta$, then $\alpha$ attacks all super-arguments of $\beta$; if $\alpha$ defeats $\beta$,  $\alpha$ defeats all super-arguments of $\beta$.
\end{proposition}

\begin{proof}
When $\alpha$ attacks $\beta$, according to Definition \ref{defeat}, $\beta$ has a sub-argument $\beta'$ such that 
$concl(\alpha)=\overline{concl(\beta')}$. Let $\gamma$ be an super-argument of $\beta$. It follows that $\beta'$ is a sub-argument of $\gamma$. Hence, $\alpha$ attacks $\gamma$. 

When $\alpha$ defeats $\beta$, according to Definition \ref{defeat}, $\beta$ has a sub-argument $\beta'$ such that 
$concl(\alpha)=\overline{concl(\beta')}$ and  
$\alpha$ is a context argument; or $\alpha$ is an ordinary argument and $ \alpha \succ \beta'$, 
Since $\beta'$ is a sub-argument of $\gamma$, $\alpha$ defeats $\gamma$. 
\end{proof}





Second, 
corresponding to properties of sub-argument closure and direct consistency in $\mbox{\it ASPIC}^+$~\cite{DBLP:journals/ai/ModgilP13}, we have the following two properties. 

\begin{proposition}[Closure under sub-arguments] \label{subarg}
Let $\mathcal{F} = (\mathcal{A}, \mathcal{R})$ be an AF constructed from a hierarchical abstract normative system. For all $\mbox{\it sem} \in \{\mbox{\it cmp}, \mbox{\it prf}, \mbox{\it grd}, \mbox{\it stb}\}$, $\forall E\in \mbox{\it sem}(\mathcal{F})$, if an argument $\alpha\in E$, then $\mbox{\it sub}(\alpha)\subseteq E$.
\end{proposition}

\begin{proof}
For every $\beta\in \mbox{\it sub}(\alpha)$, since $\alpha$ is acceptable with respect to $E$, it holds that $\beta$ is acceptable with respect to $E$. This is because for each $\gamma\in \mathcal{R}$, if $\gamma$ defeats $\beta$, then according to Proposition \ref{prop-defeat-sup}, $\gamma$ defeats $\alpha$; since $\alpha\in E$, there exists an $\eta\in E$ such that $\eta$ defeats $\gamma$.  Given that $\beta$ is acceptable with respect to $E$ and $E$ is a complete extension, it holds that $\beta\in E$.
\end{proof}

Since all norms in a hierarchical abstract normative system are defeasible, we only need to discuss direct consistency and contextual consistency.  



\begin{proposition}[Direct consistency] \label{consistency} 
Let $\mathcal{F} = (\mathcal{A}, \mathcal{R})$ be an AF constructed from a hierarchical abstract normative system. For all $\mbox{\it sem} \in \{\mbox{\it cmp}, \mbox{\it prf}, \mbox{\it grd}, \mbox{\it stb}\}$, $\forall E\in \mbox{\it sem}(\mathcal{F})$, $\{concl(\alpha)\mid \alpha \in E, \alpha \mbox{ is an ordinary argument}\} $ is consistent. 
%
\end{proposition}

\begin{proof}
Assume that there exist $\alpha,\beta\in E$ such that $concl(\alpha) = \overline{concl(\beta)}$. Since both $\alpha$ and $\beta$ are ordinary arguments, $\alpha$ attacks $\beta$, and $\beta$ attacks $\alpha$. If $\alpha\succeq \beta$ then $\alpha$ defeats $\beta$. Otherwise, $\beta$ defeats $\alpha$. In both cases, $E$ is not conflict-free, contradicting the fact that $E$ is a complete extension.  
\end{proof}

\begin{proposition}[Contextual consistency] \label{cconsistency} 
Let $\mathcal{F} = (\mathcal{A}, \mathcal{R})$ be an AF constructed from a hierarchical abstract normative system $\mathcal{H} = \langle L,N,\C,r\rangle$. For all $\mbox{\it sem} \in \{\mbox{\it cmp}, \mbox{\it prf}, \mbox{\it grd}$, $\mbox{\it stb}\}$, $\forall E\in \mbox{\it sem}(\mathcal{F})$, $\C\cup\{concl(\alpha)\mid \alpha \in E, \alpha \mbox{ is an ordinary argument}\} $ is consistent. 
%
\end{proposition}

\begin{proof}
Since $\C$ is consistent, we only need to verify that for all $a\in \C$, for all $\alpha\in E$, $\{a, concl(\alpha)\}$ is consistent. We use proof by contradiction. Assume the contrary, i.e., $concl(\alpha) = \overline{a}$. It follows that $a$ defeats $\alpha$, and therefore $\alpha\notin E$. Contradiction. So the assumption is false, i.e., $\C\cup\{concl(\alpha)\mid \alpha \in E, \alpha \mbox{ is an ordinary argument}\} $ is consistent. This completes the proof. 
\end{proof}

In the next sections, 
we present representation results for the Greedy, Reduction and Optimization approaches introduced in Section~\ref{sect-pns}, identifying equivalences between these approaches and the argument semantics based descriptions of a hierarchical abstract normative system. 


\section{Representation results for Greedy }\label{sec:greedy}
%
Based on the idea introduced in Section 1, for a totally ordered hierarchical abstract normative system, 
we have the following lemma and proposition.

\begin{lemma}[Unique extension of Greedy]\label{lemma-acyclic}
Given a totally ordered hierarchical abstract normative system $\mathcal{H} = \langle L$, $N,  \C , r\rangle$ and the corresponding argumentation framework $\mathcal{F} = (\mbox{\it Arg}(\mathcal{H})$, $ \mbox{\it Def}(\mathcal{H},\succeq_w))$, it holds that 
$\mathcal{F}$ is acyclic, and therefore has a unique extension under stable semantics.
\vskip \baselineskip
\noindent
{\bf Proof.} 
Since $\mathcal{H}$ is totally ordered, under $\succeq_w$, the relation $\succeq_w$ among arguments is acyclic. 
Assume the contrary. Then, there exist three distinct $\alpha, \beta, \gamma\in \mbox{\it Arg}(\mathcal{H})$ such that $\alpha \succeq_w \beta$, $\beta \succeq_w \gamma$ and $\gamma \succeq_w \alpha$. According to Definition  \ref{def-w-l}, 
when $\mathcal{H}$ is totally ordered, it holds that $\alpha \succ_w \beta$, $\beta \succ_w \gamma$ and $\gamma \succ_w \alpha$. According to Proposition \ref{prop-transitive}, $\alpha \succ_w \gamma$, contradicting $\gamma \succ_w \alpha$. 
Hence, $\mathcal{F}$ is acyclic, and therefore has a unique extension under stable semantics.
\end{lemma}

\begin{proposition}[Greedy is weakest link]\label{pro-w}
Given a totally ordered hierarchical abstract normative system $\mathcal{H} = \langle L$, $N,  \C , r\rangle$ and the corresponding argumentation framework $\mathcal{F} = (\mbox{\it Arg}(\mathcal{H})$, $ \mbox{\it Def}(\mathcal{H},\succeq_w))$, it holds that $\{\mbox{\it Greedy}(\mathcal{H}) \} = \Outfamily(\mathcal{F}, \mbox{\it stb})$. 
\vskip \baselineskip
\noindent
{\bf Proof.} 
Our proof is constructive: given the Greedy extension, we show how to construct a stable extension of the argumentation framework whose conclusions coincide with the Greedy extension. Since the argumentation framework has only one extension (Lemma 4.2), this completes the proof.

By Definition 2.4., $\mbox{\it Greedy}(\mathcal{H})$ is the set of elements $x$ such that there is a path in $\mathcal{H} $ from an element in $\C$ to $x$ with respect to a set of norms $R\subseteq N$, inductively defined by $R =  \cup^\infty_{i=0} R_i$, where
$R_0= \emptyset$ and  $R_{i+1} =  R_i\cup \mathrm{max}(N,\C, R_i,r)$, $\mathrm{max}(N,\C, R_i,r) = \{u\in \mbox{\it Appl}(N,  \C, R_i) \mid \forall v\in  \mbox{\it Appl}(N,  \C, R_i): r(u) \geq r(v)\}$ and $\mbox{\it Appl}(N, \C, R) := \{(a,x)\in N\setminus R \mid a\in \C \cup$ $\mbox{\it cons}(R)$, $\{x, \overline{x}\}$$\not\subseteq \C \cup \mbox{\it cons}(R)\}$.

Let  $E = \{a \in  \mbox{\it Arg}$  $(\mathcal{H})  \mid a\in \C \mbox{ is a context argument}\}\cup \{[u_1, \dots, u_n]  \in \mbox{\it Arg}(\mathcal{H})\mid  n\ge 1,   \mbox{\it ant}(u_1)\in \C, \{\mbox{\it ant}(u_2)\dots, \mbox{\it ant}(u_n),\mathit{cons}(u_n)\}\subseteq \mbox{\it Greedy}(\mathcal{H}) \}$. 
From the consistency of the extensions (Proposition 3.11) and the construction of Greedy we know that there is a consistent path to any element of $\mbox{\it Greedy}(\mathcal{H})$, and thus for any argument of $\mbox{\it Greedy}(\mathcal{H}) $ there is an ordinary argument with that element as its conclusion, and thus $\mbox{\it Greedy}(\mathcal{H})  = \{concl(\alpha)\mid \alpha \in E, \alpha \mbox{ is an ordinary argument} \}$.
Now we only need to prove that $E$ is a stable extension under weakest link.



We use proof by contradiction. Assume that $E$ is not a stable extension. Given that $\mbox{\it Greedy}(\mathcal{H})$ is consistent, and thus the extension $E$ is conflict free, and given that context arguments defeat all conflicting ordinary arguments and are thus included in $E$, there must be $\alpha = [u_1, \dots, u_n]\in  \mbox{\it Arg}(\mathcal{H})\setminus E$, such that $\alpha$ is not defeated by any argument in $E$.  Since $\alpha\notin E$, there exist $\alpha^\prime = [u_1, \dots, u_{j-1}] \in E$ and $\alpha^{\prime\prime} = [u_1, \dots, u_{j}]\not \in E$. Let $S \subseteq E$ be the set of arguments conflicting with $\alpha''$, thus each of them has a conclusion $\overline{\mathit{cons}(u_{j})}$. So, $\alpha^{\prime\prime}$ defeats each argument $\beta = [v_1,\dots, v_m]$ in $S$ using weakest link, i.e. $\alpha''\succeq_{w} \beta$.
If $\Phi_1 = \{u_1, \dots, u_n\}$ and $\Phi_2 = \{v_1, \dots, v_m\}$,
then according to Def 3.2, $\exists v\in \Phi_2\setminus \Phi_1$ such that $\forall u\in \Phi_1\setminus \Phi_2, v \le u$. 
Let $M\subseteq N$ be the set of all these norms $v$, i.e. all norms that occur in the arguments of $S$ which are smaller than all the norms in $\alpha''$, excluding norms that occur in $\alpha''$ itself.

We now consider the point in the construction of the greedy extension where an element of $M$ was added to it, we consider $\mbox{\it Appl}(N, \C, R)$ and $\mathrm{max}(N,\C, R_i,r)$, and we derive the contradiction. Let $i$ be the lowest index such that $R_{i}$ does not contain a norm of $M$. At this moment, $\mbox{\it Appl}(N, \C, R_i)$ also contains a norm in $\alpha''$, and $\mathrm{max}(N,\C, R_i,r)$ contains an element of $M$. However, by definition, all norms of $M$ are ranked lower than the norms in $\alpha''$. Contradiction. So the assumption is false, i.e. $E$ is a stable extension. This completes the proof.
\end{proposition}

Note that Proposition \ref{pro-w} corresponds to Theorem 5.3 of Young \textit{et al}.\cite{Young2016}. This correspondence arises as follows. First, in the argumentation theory for a hierarchical abstract normative system, we use disjoint elitist order  to compare sets of norms, while in the argumentation theory for prioritized default logic, Young \textit{et al.} use a new order called a structure-preference order, which takes into account the structure of how arguments are constructed. Since in the setting of hierarchical abstract normative systems, arguments are acyclic paths, it is not necessary to use the structure-preference order to compare arguments. Second, due to the different ways of constructing argumentation frameworks, the proof of Proposition \ref{pro-w}  differs from that of Theorem 5.3 of Young \textit{et al}. \cite{Young2016}. The former considers the order of the applicability of norms in the proof, while the latter uses the mechanism defined in the structure-preference order. 

\section{Representation result for Reduction}
According to Brewka and Eiter \cite{Brewka}, Reduction is based on the following two points.

\begin{description}
\item[1)] \textit{The application of a rule with nonmonotonic assumptions means jumping to a conclusion. This
conclusion is yet another assumption which has to be used \textit{globally} in the program for the issue of deciding
whether a rule is applicable or not.  }
\item[2)] \textit{The rules must be applied in an order compatible with
the priority information. }
\end{description}
\normalsize

This \textit{global} view of deciding whether a rule is applicable coincides with the last-link principle of lifting a preference relation between rules to a priority relation between resulting arguments.  According to Definition \ref{def:brewka} and the argumentation theory for a hierarchical abstract normative system, we have the following representation result. 

\begin{proposition} [Reduction is last link] \label{prol-be}
Given a totally ordered hierarchical abstract normative system $\mathcal{H} =  \langle L, N,\C,r\rangle$ and the corresponding argumentation framework $\mathcal{F} = ( \mbox{\it Arg}(\mathcal{H})$, $ \mbox{\it Def}(\mathcal{H},\succeq_l))$, it holds that $\mbox{\it Reduction}(\mathcal{H}) =  \Outfamily(\mathcal{F}, \mbox{\it stb})$.  
\vskip \baselineskip
\noindent
{\bf Proof.}
We prove $\mbox{\it Reduction}(\mathcal{H}) =  \Outfamily(\mathcal{F}, \mbox{\it stb})$. by first proving that the left hand side is a subset of the right hand side, and then by proving that the right hand side is a subset of the left hand side.

We first prove $\mbox{\it Reduction}(\mathcal{H}) \subseteq  \Outfamily(\mathcal{F}, \mbox{\it stb})$.

Assume a  set $U\in Reduction(\mathcal{H})$. According to Definition \ref{def:brewka} we have $\mbox{\it Greedy}(\mathcal{H}^U)$ where $\mathcal{H}^X$=$\langle L, N'$,  $\C$, $r'\rangle$, $N' = \{(\top, x_2) \mid (x_1, x_2)\in N, x_1\in \C \cup X\}$ is the set of ordinary norms, and $r'(\top, x_2) = \max(r(x_1, x_2)\mid   (x_1, x_2)\in N, x_1\in \C \cup X)$ for all $(x_1, x_2)\in N$ are priorities over norms.

We now construct a stable extension of $\mathcal{F}$ such that its conclusions are exactly $U$.
Let $E$ be the set of arguments that can be constructed from elements of $U$, 
just like in the proof of Greedy.
So  $E = \{a \in  \mbox{\it Arg}$  $(\mathcal{H})  \mid a\in \C \mbox{ is a context argument}\}\cup \{[u_1, \dots, u_n]  \in \mbox{\it Arg}(\mathcal{H})\mid  n\ge 1,   \mbox{\it ant}(u_1)\in \C, \{\mbox{\it ant}(u_2)\dots, \mbox{\it ant}(u_n),\mathit{cons}(u_n)\}\subseteq U \}$. 
From the consistency of the extensions (Proposition 3.11) and the construction of Reduction we know that there is a consistent path to any element of $U$, and thus for any argument of $U$ there is an ordinary argument with that element as its conclusion, and thus $U = \{concl(\alpha)\mid \alpha \in E, \alpha \mbox{ is an ordinary argument} \}$.
Now we only need to prove that $E$ is a stable extension.

We use proof by contradiction, just like in the proof of Greedy. Assume that $E$ is not stable. Given that context arguments defeat all conflicting ordinary arguments, there must be a $\beta = [v_1, \dots, v_n]\in  \mbox{\it Arg}(\mathcal{H})\setminus E$, such that $\beta$ is not defeated by any argument in $E$.  Since $\beta\notin E$, there exists $\beta^\prime = [v_1, \dots, v_j]$ in $E$ and $\beta^{\prime\prime} = [v_1, \dots, v_{j+1}]$ not in $E$. Let $S \subseteq E$ be the set of arguments attacking $\beta''$, thus each of them has a conclusion $\overline{\mathit{cons}(v_{j+1})}$. Since $\beta''$ is not defeated by any argument in $E$, $\beta^{\prime\prime}$ defeats each argument in $S$. Then, the last link of $\beta''$ is higher than the last links of all arguments of $S$. Consequently, the rank of $v_{j+1}$ is higher than the rank of all norms with consequent $\overline{\mathit{cons}(v_{j+1})}$. But this means that in $N'$, the rank of $(\top, \mathit{cons}(v_{j+1}))$ is higher than the rank of $(\top,\overline{\mathit{cons}(v_{j+1})})$. Then by the construction of Reduction, i.e.,  $\mbox{\it Greedy}(\mathcal{H}^U)$, we have that $\mathit{cons}(v_{j+1})$ must be in $U$, a contradiction. Thus the assumption is false, i.e., $E$ is stable, and this completes the proof that $\mbox{\it Reduction}(\mathcal{H}) \subseteq  \Outfamily(\mathcal{F}, \mbox{\it stb})$.

To complete our proof, we now show that $\mbox{\it Reduction}(\mathcal{H}) \supseteq  \Outfamily(\mathcal{F}, \mbox{\it stb})$. 
Let $E$ be a stable extension of $\mathcal{F} = ( \mbox{\it Arg}(\mathcal{H})$, $ \mbox{\it Def}(\mathcal{H},\succeq_l))$.
Thus $E$ is conflict free, closed under sub-arguments, and for all arguments not in $E$, there is an argument in $E$ defeating it.
In particular, for any $\beta^\prime = [v_1, \dots, v_j]$ in $E$ and $\beta^{\prime\prime} = [v_1, \dots, v_{j+1}]$ not in $E$, there is an argument in $E$ defeating $\beta''$. 

Moreover, let $U = \{concl(\alpha)\mid \alpha \in E, \alpha \mbox{ is an ordinary argument} \}$. From the contextual consistency of stable extensions (Proposition 3.12), we know that $U\cup C$ is consistent. Since defeat is based on last link, it means that for every norm applicable in $U$ whose conclusion $c$ is not in $U$, there is a higher ranked norm applicable in $U$ whose conclusion is $\overline{c}$.  
Consider $\mathcal{H}^U$=$\langle L, N'$,  $\C$, $r'\rangle$, $N' = \{(\top, x_2) \mid (x_1, x_2)\in N, x_1\in \C \cup U\}$ is the set of ordinary norms, and $r'(\top, x_2) = \max(r(x_1, x_2)\mid   (x_1, x_2)\in N, x_1\in \C \cup U)$ for all $(x_1, x_2)\in N$ are priorities over norms. Assume there is an $x\in U$ with $(\top,x),(\top,\overline x)\in N'$. Due to the above the rank of $(\top,x)$ is higher than the rank of $(\top,\overline x)$, and thus $x\in \mbox{\it Greedy}(\mathcal{H}^U)$.

Thus $U=\mbox{\it Greedy}(\mathcal{H}^U)$, i.e. $U\in Reduction(\mathcal{H})$,  and that completes our proof.
\end{proposition}

Now, let us consider the following three examples, which show that by Reduction, a hierarchical abstract normative system and the corresponding argumentation framework may have a unique extension, multiple extensions, or an empty extension.  

\begin{example}[Order puzzle, Reduction] \label{Brewka-Eiter}
Consider Example  \ref{Ptc} when the last link principle is applied, $A_2$ defeats $A_3$. Then, we have $\Outfamily(\mathcal{F}_1,$ $\mbox{\it stb}) = \{\{h,o\}\}$, which is equal to $\mbox{\it Reduction}(\mathcal{H})$.
\end{example}

\begin{figure}[t]
\centering
\includegraphics[width=0.65\textwidth]{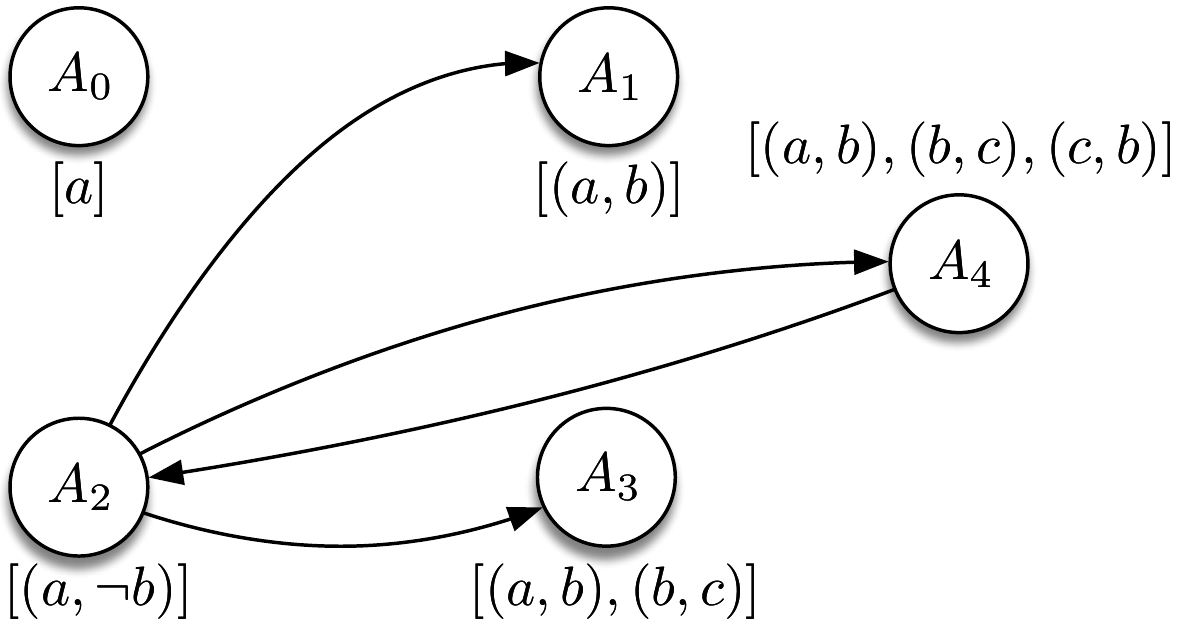}
\caption{\label{fig:poke2Arg} The argumentation framework obtained for the hierarchical abstract normative system of Figure \ref{fig:poke2}.}
\end{figure}

\begin{example}[Reduction, Multiple extensions]
Consider the hierarchical abstract normative system of Figure \ref{fig:poke2} when the last link principle is applied. We obtain the argumentation framework shown in Figure \ref{fig:poke2Arg}, written as $\mathcal{F}_3$, yielding $\Outfamily(\mathcal{F}_3,\mbox{\it stb})=\{\{b,c\},\{\neg b\}\}$. Note that here, we have two distinct stable extensions.
\end{example}

Since stable extensions do not necessarily exist for all argumentation frameworks, the Reduction of a hierarchical abstract normative system might not exist.

\begin{figure}[t]
\centering
\includegraphics[width=0.7\textwidth]{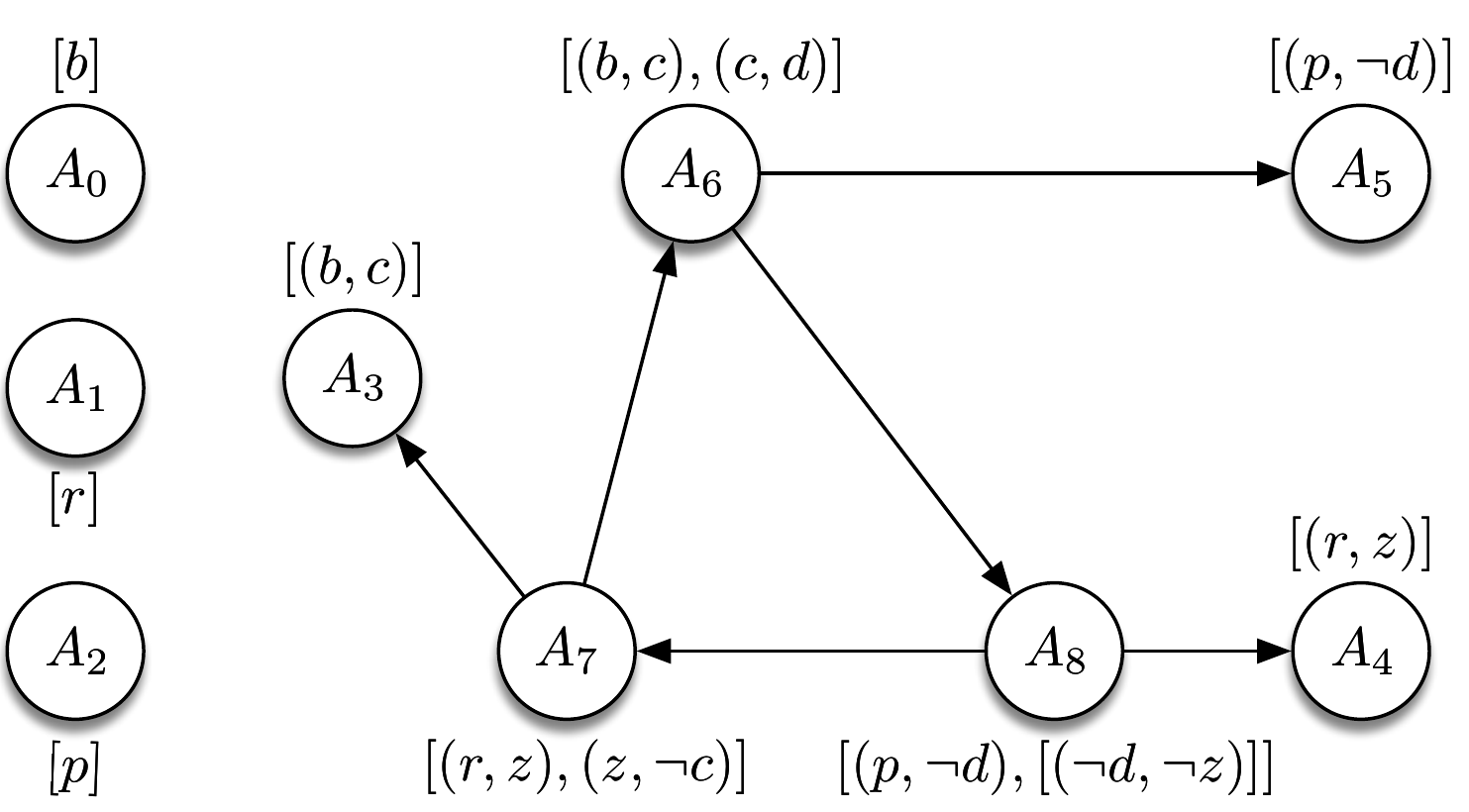}
\caption{\label{fig:noStable} An argumentation framework with no stable extension.}
\end{figure}

\begin{example}[Reduction, Empty extension]
Consider the hierarchical abstract normative system $\mathcal{H}=\langle L,N,C,r \rangle$ where $N=\{(c,d),(p,\neg d),(z,\neg c),(\neg d, \neg z), (r,z),(b,c)\}$,  $C=\{b,r,p\}$ and $r(c,d)=5$, $r(p,\neg d)=4$, $r(z,\neg c)=6$, $r(\neg d, \neg z)=2$, $r(r,z)=1$ and $r(b,c)=0$. When the last link principle is applied, this hierarchical abstract normative system yields the argumentation framework shown in Figure \ref{fig:noStable}, which has no stable extension. 
\end{example}

\section{Representation result for Optimization}\label{sec:minimal}


In principle, Optimization 
is realized by adding norms in order of priority which are consistent with the context, until no more norms can be added, obtaining a maximal set of obeyable norms, following which conclusions can be computed. For each norm that does not belong to the maximal set of obeyable norms, if it is in a consistent path of a hierarchical abstract normative system $\mathcal{H}$, it is the weakest norm of this path.
In terms of the terminology of argumentation, a consistent path of $\mathcal{H}$ is an argument of the corresponding argumentation framework. Hence, for an ordinary argument, if its top norm is the weakest norm that does not belong to the maximal obeyable set, then this argument should not be accepted. In other words, this argument should be defeated by some accepted argument. Furthermore, when an argument is rejected, all its super-arguments should be rejected.  
For instance, as illustrated in Figure 1, the maximal obeyable set of norms is $\{(h, o), (w, \neg o)\}$. The norm $(w, h)$ does no belong to this set. So, the arguments containing $(w, h)$, i.e., $[(w,h)]$ and $[(w,h), (h,o)]$, should be rejected.

In this paper, given an argument $\alpha$, a sub-argument of $\alpha$ is called the \textit{weakest sub-argument} of $\alpha$, if it is an ordinary argument and its top norm is the weakest norm of $\alpha$. 

\begin{definition}[Weakest sub-argument]\label{ws}
Given a totally ordered hierarchical abstract normative system $\mathcal{H} = \langle L$, $N,  \C , r\rangle$ and the corresponding argumentation framework $\mathcal{F} = (\mbox{\it Arg}(\mathcal{H})$, $ \mbox{\it Def}(\mathcal{H},\succeq_w))$, for all $\alpha = [u_1, \dots, u_n]\in \mbox{\it Arg}(\mathcal{H})$, if $u_i$ is the weakest norm of $\alpha$, then the ordinary sub-argument of $\alpha$ whose top norm is $u_i$ is called the weakest sub-argument of $\alpha$. 
\end{definition}

Then, the set of super-arguments of $\alpha^\prime$ ---the weakest sub-argument of $\alpha$--- is called the weakest arguments with respect to $\alpha$, formally defined as follows.

\begin{definition}[Weakest argument]\label{wsub}
Given a totally ordered hierarchical abstract normative system $\mathcal{H} = \langle L$, $N,  \C , r\rangle$ and the corresponding argumentation framework $\mathcal{F} = (\mbox{\it Arg}(\mathcal{H})$, $ \mbox{\it Def}(\mathcal{H},\succeq_w))$, for all $\alpha = [u_1, \dots, u_n]\in \mbox{\it Arg}(\mathcal{H})$, let $\alpha^\prime$ be the weakest sub-argument of $\alpha$. Then, the set of weakest arguments with respect to $\alpha$ is defined as $\mathit{warg}(\alpha) = \mathit{sup}(\alpha^\prime)$.
\end{definition}

The notion of weakest arguments can be illustrated by the following example.

\begin{example}[Order puzzle, weakest arguments]
Consider the argumentation framework in Figure 2. When applying the weakest link principle, $A_3$ defeats $A_2$. According to Definition \ref{wsub}, $\mathit{warg}(A_1) = \{A_1\}$, $\mathit{warg}(A_2) = \{A_1, A_2\}$, and $\mathit{warg}(A_3) = \{A_3\}$.
\end{example}

Based on the concept of weakest arguments, when an argument $\alpha$ defeats another argument $\beta$ according to the weakest link principle, there are two different cases, namely that $\alpha$ is in the set of weakest arguments with respect to $\beta$, or that it is not.

First, consider the case where $\alpha$ is in $\mathit{warg}(\beta)$.  In this case, it means that the weakest norm of $\beta$ is also the weakest norm of $\alpha$. 

\begin{proposition}[Shared weakest norm]
Given $\mathcal{F} = (\mbox{\it Arg}(\mathcal{H})$, $ \mbox{\it Def}(\mathcal{H},\succeq_w))$, for all $(\alpha, \beta) \in \mbox{\it Def}(\mathcal{H},\succeq_w)$, if $u$ is the weakest norm of $\beta$ and $\alpha\in \mathit{warg}(\beta)$, then $u$ is the weakest norm of $\alpha$. 
\end{proposition}

\begin{proof}
We use proof by contradiction. Assume that $u$ is not the weakest norm of $\alpha$. Then, there exists $u^\prime$ such that $u > u^\prime$. Since $u^\prime$ does not belong to $\beta$, it holds that $\beta \succ_w \alpha$. This contradicts $(\alpha, \beta) \in \mbox{\it Def}(\mathcal{H},\succeq_w)$. Thus the assumption is false, i.e., $u$ is the weakest norm of $\alpha$. 
\end{proof}

Given that $\alpha$ and $\beta$ share a weakest norm, according to the definition of Optimization, this weakest norm does not belong to the maximal obeyable set. Hence, both $\alpha$ and $\beta$ cannot be accepted. In this paper, we use an auxiliary argument $\mathrm{aux}$ to defeat all arguments in $\mathit{warg}(\beta)$ that should not be accepted. 

The second case occurs when $\alpha$ is not in $\mathit{warg}(\beta)$. Here, the weakest norm of $\alpha$ is superior to the weakest norm of $\beta$. So, if $\alpha$ is accepted, then all arguments in $\mathit{warg}(\beta)$ should be rejected. For this purpose, we may use $\alpha$ to defeat each argument in $\mathit{warg}(\beta)$. In addition, since the defeat from $\alpha$ to each super-argument of $\beta$ is already in $\mbox{\it Arg}(\mathcal{H})$, only the arguments in $\mathit{warg}(\beta)\cap \mathit{psub}(\beta)$ should be added to $\mbox{\it Arg}(\mathcal{H})$. 

According to the above observations, an expanded argumentation framework of $\mathcal{F}$ with auxiliary defeats on weakest arguments is defined as follows.  

\begin{definition}[Expanded argumentation framework with additional defeats on weakest arguments]\label{def-ean}
Let $\mathcal{F} = (\mbox{\it Arg}(\mathcal{H})$,$ \mbox{\it Def}(\mathcal{H},\succeq_w))$ be an argumentation framework  that is constructed from a totally ordered hierarchical abstract normative system $\mathcal{H} = \langle L$, $N,  \C , r\rangle$, and $\mathrm{aux}$ be an auxiliary argument such that $\mathrm{aux}\notin \mbox{\it Arg} (\mathcal{H})$. The expanded argumentation framework of $\mathcal{F}$ with auxiliary defeats on weakest arguments is $\mathcal{F}^\prime = (\mbox{\it Arg}(\mathcal{H})\cup \{\mathrm{aux}\}$,  $ \mbox{\it Def}(\mathcal{H},\succeq_w)\cup \Phi_1\cup \Phi_2)$ where $\Phi_1 = \cup_{(\alpha,\beta)\in \mbox{\it Def}(\mathcal{H},\succeq_w)\wedge  \alpha\notin \mathit{warg}(\beta)} \{(\alpha, \gamma)\mid \gamma\in \mathit{warg}(\beta)\cap \mathit{psub}(\beta)\}$, and $\Phi_2 = \cup_{ (\alpha,\beta)\in \mbox{\it Def}(\mathcal{H},\succeq_w)\wedge  \alpha\in \mathit{warg}(\beta)} \{(\mathrm{aux},  \gamma)\mid \gamma\in \mathit{warg}(\beta)\}$. 
\end{definition}

In Definition \ref{def-ean}, $\Phi_1$ is the set of defeats from $\alpha$ that is not a weakest argument with respect to $\beta$, and $\Phi_2$ is the set of defeats from the auxiliary argument $\mathrm{aux}$ when $\alpha$ is a weakest argument with respect to $\beta$.  


The following lemma and proposition show that Optimization can be represented in formal argumentation by using weakest link  together with auxiliary defeats.

\begin{lemma}[Unique extension of Optimization]\label{lem-acy-ex}
Let $\mathcal{H} = \langle L$, $N,  \C , r\rangle$ be a totally ordered hierarchical abstract normative system, and  $\mathcal{F}^\prime$ be an argumentation framework of $\mathcal{H}$ with additional defeats on weakest arguments presented in Definition \ref{wsub}. It holds that 
$\mathcal{F}^\prime$ is acyclic, and therefore has a unique extension under stable semantics.
\vskip \baselineskip
\noindent
{\bf Proof.} 
According to Lemma \ref{lemma-acyclic}, $\mathcal{F}$ is acyclic. The addition of $\Phi_2$ does not produce cycles. Meanwhile, for all $(\alpha, \gamma)\in \Phi_1$, since $\alpha$ defeats $\beta$ and $\gamma$ is a weakest argument with respect to $\beta$,  $\alpha\succeq_w \gamma$. So, the addition of $\Phi_2$ does not produce cycles either. As a result, $\mathcal{F}^\prime$ is acyclic, and therefore has a unique extension under stable semantics.
\end{lemma}

\begin{proposition}[Optimization is weakest link plus auxiliary defeats]\label{prop:wcp}
Let $\mathcal{H} = \langle L$, $N,  \C , r\rangle$ be a totally ordered hierarchical abstract normative system, and  $\mathcal{F}^\prime$ be an argumentation framework of $\mathcal{H}$ with additional defeats on weakest arguments presented in Definition \ref{wsub}. 
It holds that $\mbox{\it Optimization}(\mathcal{H})  =  \{concl(\alpha)\mid \alpha \in E\setminus\{\mathrm{aux}\}, \alpha \mbox{ is an ordinary argument} \} $  where $E$ is the unique stable  extension of $\mathcal{F}^\prime$.


\vskip \baselineskip
\noindent
{\bf Proof.} 
This proof is similar to that of Greedy: given the Optimization extension, we show how to construct a stable extension of the expanded argumentation framework whose conclusions coincide with the Optimization extension. Since the argumentation framework has only one extension (Lemma \ref{lem-acy-ex}), this completes the proof.

By Definition \ref{minviolation}, $\mbox{\it Optimization}(\mathcal{H})$ is the set of elements $x$ such that there is a path in $\mathcal{H}$ from an element in $\C$ to $x$ with respect to a set of norms $R\subseteq N$. Given $T = (u_1, u_2, \dots, u_n)$ the linear order on $N$ such that $u_1>u_2>\dots>u_n$, $R$ is inductively defined by $R = R_n$, where
$R_0= \emptyset$;  $R_{i+1} =  R_i\cup \{u_i\}$ if $C \cup R(C)\mbox{ is consistent}$ where $R=R_i \cup \{u_i\}$, and $R_{i+1} = R_i$ otherwise. 

Let  $E = \{\mathrm{aux}\}\cup \{\alpha \in  \mbox{\it Arg}$  $(\mathcal{H})  \mid \alpha\in \C \mbox{ is a context argument}\}\cup \{[u_1, \dots, u_n]  \in \mbox{\it Arg}(\mathcal{H})\mid  n\ge 1,   \mbox{\it ant}(u_1)\in \C, \{\mbox{\it ant}(u_2)\dots, \mbox{\it ant}(u_n),\mathit{cons}(u_n)\}\subseteq \mbox{\it Optimization}(\mathcal{H}) \}$. 
From the consistency of the extensions (Proposition 3.11) and the construction of Optimization we know that there is a consistent path to any element of $\mbox{\it Optimization}(\mathcal{H})$, and thus for any argument of $\mbox{\it Optimization}(\mathcal{H}) $ there is an ordinary argument with that element as its conclusion, and thus $\mbox{\it Optimization}(\mathcal{H})  = \{concl(\alpha)\mid \alpha \in E\setminus \{\mathrm{aux}\}, \alpha \mbox{ is an ordinary argument} \}$.
  
Now we only need to prove that $E$ is a stable extension of $\mathcal{F}^\prime$.
 
We use proof by contradiction. So assume that $E$ is not a stable extension. On the one hand, given that $\mbox{\it Optimization}(\mathcal{H})$ is consistent, and thus the extension $E$ is conflict free, and given that context arguments defeat all conflicting ordinary arguments and are thus included in $E$, there must be $\alpha = [u_1, \dots, u_n]\in  \mbox{\it Arg}(\mathcal{H})\setminus E$, such that $\alpha$ is not defeated by any argument in $E$.  
On the other hand, since $\alpha\notin E$, there exists a weakest norm $u_j$ of $\alpha$ such that $u_j \notin R$.  Let $i$ be the index where $u_j$ can not be added to $R_i$ since $\C \cup R(\C)$  is consistent  where $R=R_{j+1}=R_i \cup \{u_j\}$.   Let   $\alpha^\prime = [u_1, \dots, u_k]$ where $k\ge j$, and $\beta = [v_1, \dots, v_m]$ such that  $\{v_1, \dots, v_m\} \subseteq R_{i+1}$ and $v_m = \overline{u_k}$. Since $u_j<v$ for all $v\in R_i$, it holds that $\beta$ defeats $\alpha^\prime$. Then, we have the following two possible cases. If $u_j\in \{v_1, \dots, v_m\}$, then $\alpha$ is defeated by $\mathrm{aux}\in E$. Contradiction. Otherwise, $u_j\notin \{v_1, \dots, v_m\}$.  In this case, $\{v_1, \dots, v_m\}\subseteq R_{i}$ and $\C \cup R(\C)$  is consistent where $R= R_{i}$. So,  $\beta$ is in $E$. Contradiction. So the assumption is false, i.e., $E$ is a stable extension. This completes the proof.
\end{proposition}

Now, let us consider a revised version of the Order puzzle as follows. 

\begin{example}[Order puzzle, Optimization]\label{ex-orderp}
Let $\mathcal{H}^\prime = \langle L, N, \C, r\rangle$ be a hierarchical abstract normative system, where 
$L = \{w, h, o, \neg w, \neg h, \neg o, \top\}$, $N  = \{(w,h), (w, \neg h)$, $(h, o), (\neg h, o), (w, \neg o)\} $, $\C = \{w, \top\}$, $r(w,h) =1$,  $r(w,\neg h) =0$,   
	 $r(h,o) =3$,  $r(\neg h,o) =4$,   
	 $r(w, \neg o) =2$.  
The maximal set of obeyable norms is $R =\{(\neg h,o), (h,o)$, $(w, \neg o)\}$ and so $\mathit{Optimization}(\mathcal{H}')=\{\neg o\}$. Figure \ref{af66} illustrates the argumentation framework obtained from this hierarchical abstract normative system, by adding an auxiliary argument and two auxiliary defeats. In this example, since $A_5$ is neither in $\mathit{warg}(A_3)=\{A_1, A_3\}$ nor in $\mathit{warg}(A_4)=\{A_2, A_4\}$, the two auxiliary defeats are from $A_5$ to $A_1$ and $A_2$ respectively. Then, under stable semantics, the expanded argumentation framework has a unique extension $E= \{\mathrm{aux}, A_0, A_5\}$. As a result, the set of conclusions $\{concl(\alpha)\mid \alpha \in E\setminus\{\mathrm{aux}\}, \alpha \mbox{ is an ordinary argument} \} = \{concl(A_5)\} = \{\neg o\}$. 

\begin{figure}
\centering
\includegraphics[width=0.57\textwidth]{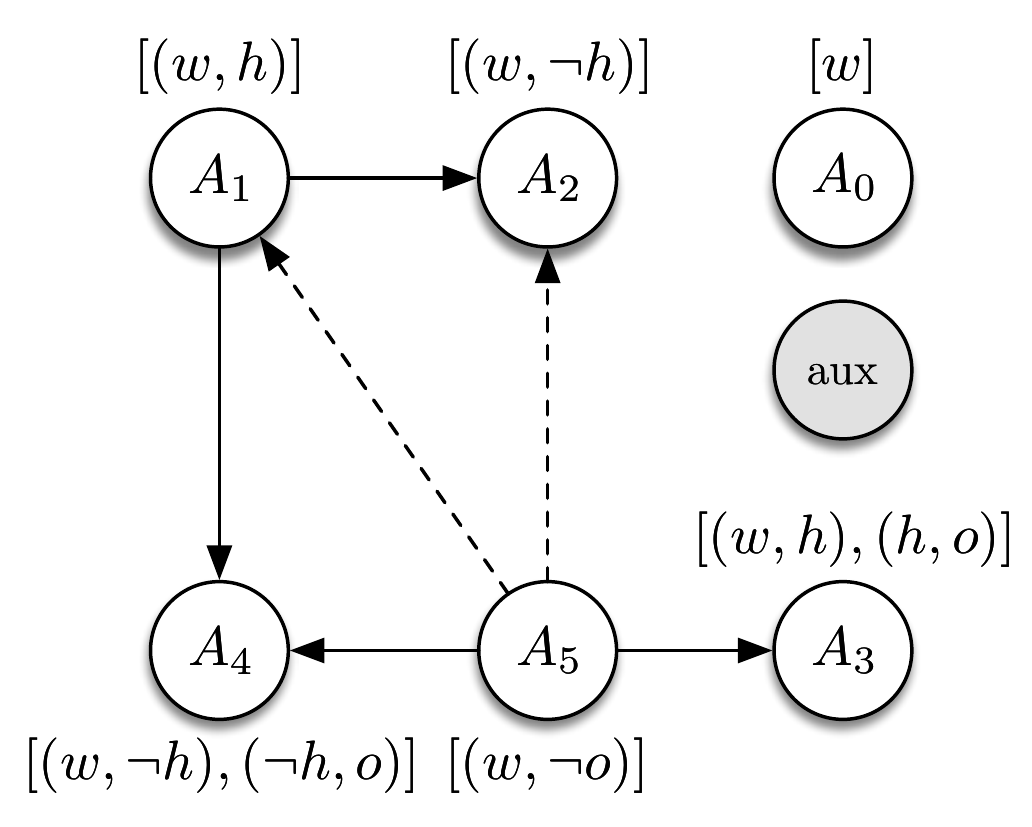}
\caption{\label{af66} Argumentation framework expanded by an auxiliary argument and two auxiliary defeats (denoted using dashed lines).}
\end{figure}
\end{example}

Finally, let us consider an expanded argumentation framework where auxiliary defeats are from the auxiliary argument $\mathrm{aux}$. 

\begin{example}[Empty Optimization]
Let $\mathcal{H} = \langle L, N$,  $\C$, $r\rangle$ be a hierarchical abstract normative system, where 
$L = \{a, b, c, \neg a, \neg b, \neg c, \top\}$, $N  = \{(a,b), (b, c), (b, \neg c)\} $, $\C = \{a, \top\}$, $r(a,b) =1$,   
	 $r(b,c) =2$,    
	 $r(b, \neg c) =3$.  
The maximal set of obeyable norms is $R =\{(b, c), (b, \neg c)\}$ and so $\mathit{Optimization}(\mathcal{H})=\emptyset$. Figure \ref{af65}  illustrates the argumentation framework obtained from this hierarchical abstract normative system, by adding an auxiliary argument and three auxiliary defeats. In this example, since $A_3$ is in $\mathit{warg}(A_2)=\{A_1, A_2, A_3\}$, the three auxiliary defeats are from $\mathrm{aux}$. Then, under stable semantics, the expanded argumentation framework has a unique extension $E= \{\mathrm{aux}, A_0\}$. As a result, the set of conclusions $\{concl(\alpha)\mid \alpha \in E\setminus\{\mathrm{aux}\}, \alpha \mbox{ is an ordinary argument} \} = \emptyset$.     

\begin{figure}
\centering
\includegraphics[width=0.45\textwidth]{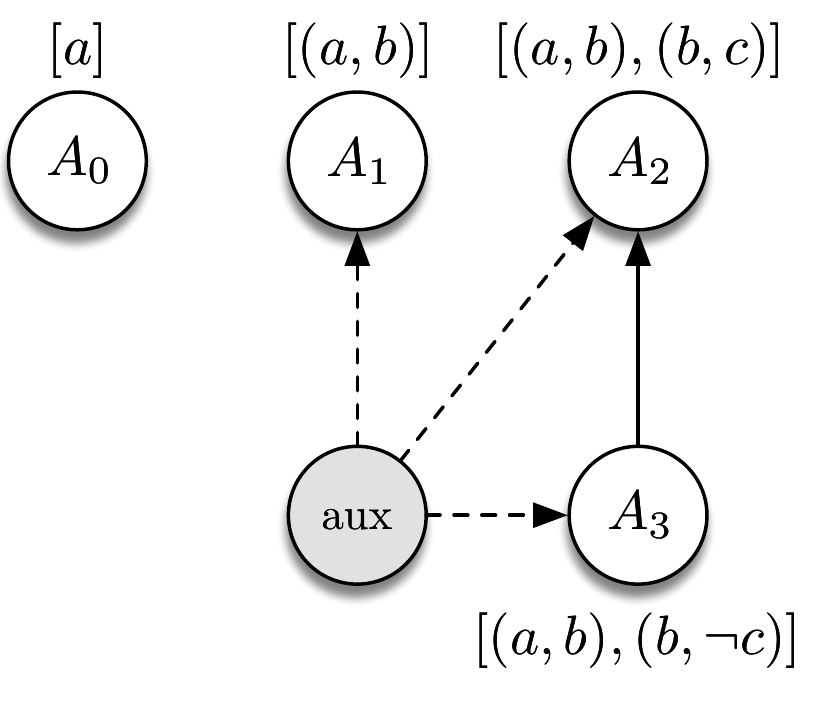} 
\caption{\label{af65} Argumentation framework expanded by an auxiliary argument and three auxiliary defeats (denoted using dashed lines).}
\end{figure}
\end{example}


\section{Discussions and related work} 

In the previous sections, after defining three detachment procedures for totally ordered hierarchical abstract normative systems motivated by the Order puzzle example, we formulated an argumentation theory to represent them. The role of examples in the study of logic has a long and rich history. Traditionally, a logic was proposed to model some example problem, following which examples were introduced to highlight paradoxes or inconsistencies in the logic, whereupon a new logic was proposed to address these problems, and the cycle repeated. While this approach has significantly enriched the field, it is not without problems. For example, there is still a debate regarding deontic detachment 
within the community, as in some cases, deontic detachment intuitively holds, and in other cases it does not \cite{nute97defeasible}. Given this, we do not seek to claim that the detachment procedures we present in this paper are in any sense the `right' logics. Instead, our goal is to answer the following questions.

\begin{enumerate}
\item What are the general properties of systems considered relevant to some problem?
\item Given an application, what choices should be made in order to obtain a solution?
\end{enumerate}

In this paper, the systems we considered are the three different detachment procedures, encoded in the general framework of hierarchical abstract normative systems. The property we considered is then the conclusions that one can draw from each of the detachment procedures in the context of prioritized norms, which we describe in the context of an argumentation system.

Our results then characterize the outputs of  Greedy, Reduction and Optimization in terms of argumentation for a totally ordered hierarchical abstract normative system, allowing one to decide which approach is relevant to their needs by understanding the effects of each approach through the argumentation literature. The semantics associated with each approach also sheds light on the complexity of computing conclusions in the normative context.

Furthermore, it is important to note that our representation results only hold when the hierarchical abstract normative system is totally ordered. This is illustrated by the following example, which shows that Greedy  does not match the results of weakest link when one does not have a total order over preferences.

Let $\mathcal{H}=\langle L, N, \C, r\rangle$ be  a hierarchical abstract normative system, where 
 $N=\{(a,b),(a,c),(b,\neg c),(c, \neg b)\}$
  $C=\{a\}$ and
 $r(a,b)=1,r(a,c)=1,r(b,\neg c)=2$ and $r(c, \neg b)=2\}$. On the one hand, by Greedy, there are two extensions $\{b, \neg c\}$ and $\{c, \neg b\}$.
On the other hand, by the weakest link principle, the argumentation framework constructed from $\mathcal{H}$ is illustrated in Figure \ref{ex-example1-AF-1}. Under stable semantics, there are three extensions $\{A_0, A_1, A_2\}$, $\{A_0, A_1, A_3\}$, and $\{A_0, A_3, A_4\}$. So, there are three conclusion extensions $\{b, \neg c\}$, $\{b, c\}$, $\{c, \neg b\}$. As a result, the set of conclusions obtained by Greedy is not equal to the one obtained by argumentation using the weakest link principle. We leave identifying valid representation results for hierarchical abstract normative systems containing preference preorders as an avenue for future research.  
\begin{figure}
\centering
\includegraphics[width=0.5\textwidth]{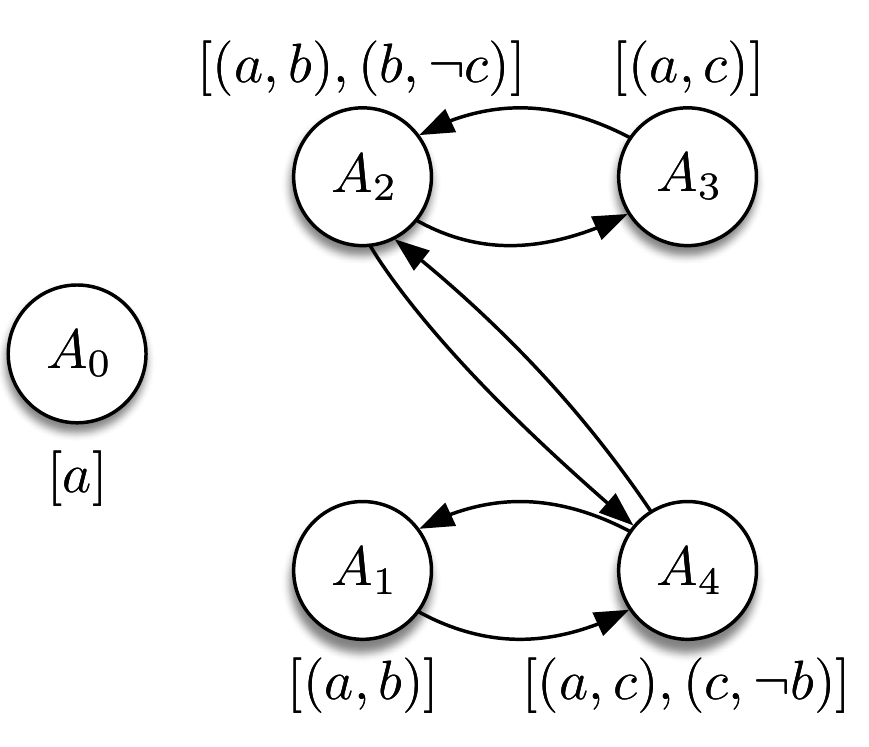}
\caption{\label{ex-example1-AF-1} Argumentation framework constructed from a preordered hierarchical abstract normative system.}
\end{figure}

Regarding related work, Young \textit{et al.}~\cite{Young2016} endowed Brewka's prioritized default logic (PDL) with argumentation semantics using the ASPIC$^+$ framework for structured argumentation~\cite{DBLP:journals/argcom/ModgilP14}. 
More precisely, their goal is to define a preference ordering over arguments $\succsim$, based on the strict total order over defeasible rules defined to instantiate $ASPIC^+$ to PDL, so as to  ensure that an extension within PDL corresponds to the justified conclusions of its $ASPIC^+$ instantiation. Several options are investigated, and they demonstrate that the standard $ASPIC^+$ \textit{elitist} ordering cannot be used to calculate $\succsim$ as there is no correspondence between the argumentation-defined inferences and PDL, and the same holds for a  disjoint elitist preference ordering. The authors come up with a new argument preference ordering definition which captures both preferences over arguments and also \textit{when} defeasible rules become applicable in the arguments' construction, leading to the definition of a strict total order on defeasible rules and corresponding non-strict arguments. Their representation theorem shows that a correspondence always exists between the inferences made in PDL and the conclusions of justified arguments in the $ASPIC^+$ instantiation under stable semantics. 

Brewka and Eiter~\cite{Brewka} consider programs supplied with priority information, which is given by a supplementary strict
partial ordering of the rules. This additional information is used to solve potential conflicts. Moreover, their idea is that conclusions should be only those literals that are contained in at least one answer
set. They propose to use preferences on rules for selecting
a subset of the answer sets, called the \textit{preferred answer sets}. In their approach, a rule is applied unless it is defeated via
its assumptions by rules of higher priorities.
Our definition (Def.~\ref{def:brewka}) and the original formalism of Brewka and Eiter~\cite{Brewka} are different, in the sense that in our definition we do not make use of default negation to represent the exceptions, i.e., the defeasibility, of a (strict) rule. Rather, we use defeasible rules and the notion of the  applicability of such rules. This means that the correct translation of the Order Puzzle of  Example~\ref{pans-1} ends up with the following logic program\footnote{Note that in Nrewka and Eiter's original formulation \cite{Brewka} $\mathtt{r_0 < r_3}$ means that $\mathtt{r_0}$ has higher priority than $\mathtt{r_3}.$}:\\



\noindent $\mathtt{r_0: w.}\\
\mathtt{r_1: h\ \operatorname{:-}\ not\ \neg h, w.}\\
\mathtt{r_2: \neg o\ \operatorname{:-}\ not\  o, w.}\\
\mathtt{r_3: o\ \operatorname{:-}\ not\  \neg o, h.}\\
\mathtt{r_0 < r_3 < r_2 < r_1}\\$

\noindent If preferences are disregarded, then this logic program has two answer sets: $\{w, h, \neg o\}$ and $\{w, h, o\}$. Thus, considering preferences, the latter is the unique preferred answer set. After dropping the context $w$ from the answer set, we get an extension $\{h, o\}$, which is identical to the result obtained in Example~\ref{pans-1c}. 


Dung~\cite{DBLP:journals/ai/Dung16} presents an approach to deal with contradictory conclusions in defeasible reasoning with priorities. More precisely, he starts from the observation that often, the proposed approaches to defeasible reasoning with priorities (e.g.,~\cite{DBLP:conf/ijcai/Brewka89,DBLP:conf/ijcai/SchaubW01,DBLP:journals/ai/ModgilP13}) sanction contradictory conclusions, as exemplified by  $ASPIC^+$ using the weakest link principle together with the elitist ordering which returns contradictory conclusions with respect to its other three attack relations, and the conclusions reached with the well known approach of Brewka and Eiter~\cite{Brewka}. Dung shows then that the semantics for any complex interpretation of default preferences can be characterized by a subset of the set of stable extensions with respect to the normal attack relation assignments, i.e., a normal form for ordinary attack relation assignments. 
In the setting of this paper, the notion of `normal attack relation' could be defined as follows. Let $\alpha = (a_1, \dots, a_n)$ and $\beta= (b_1, \dots, b_m)$ be arguments constructed from a hierarchical abstract normative system. Since we have no Pollock style undercutting argument (as in $\mbox{\it ASPIC}^+$) and each norm is assumed to be defeasible, $\alpha$ is said to normally attack  argument $\beta$ if and only if $\beta$ has a sub-argument $\beta'$ such that $concl(\alpha) = \overline{concl(\beta')}$, and $r((a_{n-1}, a_n))\geq r((b_{m-1}, b_m))$. According to Definitions~\ref{def-w-l} and~\ref{defeat}, the normal defeat relation is equivalent to the defeat relation using the last link principle in this paper.

Kakas \textit{et al.}~\cite{DBLP:conf/cilc/KakasTM14} present a logic of arguments called \textit{argumentation logic}, where the foundations of classical logical reasoning are represented from an argumentation perspective. More precisely, their goal is to integrate into the single argumentative representation framework both classical reasoning, as
in propositional logic, and defeasible reasoning. 

You \textit{et al.}~\cite{DBLP:journals/tkde/YouWY01} define a prioritized argumentative characterization of non-monotonic reasoning, by casting default reasoning as a form of prioritized argumentation. They illustrate how the parameterized formulation of priority may be used to allow various extensions and modifications to default reasoning.

We, and all these approaches, share the idea that an argumentative characterization of NMR formalisms, like prioritized default logic in Young's case and hierarchical abstract normative systems in our approach, contributes to make the inference process more transparent to humans. However, the targeted NMR formalism is different, leading to different challenges in the representation results. To the best of our knowledge, no other approach addressed the challenge of an argumentative characterization of prioritized normative reasoning. 

The reason  we study prioritized normative reasoning \textit{in the setting of formal argumentation} is twofold. First, formal argumentation has been recognized as a popular research area in AI, thanks to 
its ability to make the inference process more intuitive and provide natural explanations for the reasoning process~\cite{caminada14scrutable}; its flexibility in dealing with the dynamics of the system; and its appeal in sometimes being more computationally efficient than competing approaches. Second, while some progress has been made on the use of priorities within argumentation (e.g.,~\cite{DBLP:conf/sum/AmgoudV10,DBLP:journals/ai/ModgilP13}), 
how to  represent different  approaches  for  prioritized normative reasoning  in argumentation  is  still a challenging issue.

\section{Conclusions}
In this paper we embedded three approaches to prioritized normative reasoning---namely the Greedy~\cite{Young2016}, Reduction~\cite{DBLP:conf/ijcai/Brewka89} and Optimization~\cite{DBLP:journals/aamas/Hansen08} approaches---within  the framework of a hierarchical abstract normative system. Within such a system, conditional norms are represented by a binary relation over literals, and priorities are represented by natural numbers. Hierarchical abstract normative systems provide an elegant visualisation of a normative system, with conflicts shown as two paths to a proposition and its negation. Since both conflicts and exceptions can be encoded, such systems  are inherently non-monotonic. In Dung~\cite{DBLP:journals/ai/Dung95}, the author pointed out that ``many of the major approaches to nonmonotonic reasoning in AI and logic programming are different forms of argumentation'', and inspired by this, we described how arguments can be instantiated as paths through a hierarchical abstract normative system; demonstrated that this instantiation satisfies certain desirable properties; and described how attacks and defeats between these arguments can be identified. Defeats in particular are dependent on the priorities associated with the arguments, and several different techniques have been proposed to lift priorities from argument components --- made up of norms in the context of our work --- to the arguments themselves \cite{DBLP:journals/ai/ModgilP13}. We demonstrated that for a total ordering of priorities, lifting priorities to arguments based on the weakest link principle, evaluated using the stable semantics, is equivalent to Greedy; that lifting priorities to arguments based on last link and using the stable semantics is equivalent to Reduction; and that the Optimization approach can be encoded by an argumentation system which uses weakest link together with the stable semantics, and which introduces additional defeats capturing implicit conflicts between arguments.  

%
This last result---which requires a relatively complex argumentative representation---opens up an interesting avenue for future work, namely in determining which non-monotonic logics can be easily captured through standard formal argumentation techniques, and which require additional rules or axioms in order to be represented. We note that on the argumentation side, work on bipolar argumentation (e.g., \cite{Cayrol2015}) has considered introducing additional defeats between arguments based on some notion of support, and we intend to  investigate how the additional defeats we introduced can be categorized in such frameworks.

Apart from our representation results, the use of argumentation allows us to make some useful observations, such as that Reduction will sometimes not reach any conclusions. Furthermore, argumentation can be used to provide  explanation \cite{caminada14scrutable}. When implemented, a system building on our approach can help users understand what norms they should comply with, and why. For large normative systems, the use of stable semantics to compute Reduction and Optimization results in a high computational overhead, while Greedy is computationally efficient. Ultimately, selecting the correct reasoning procedure thus requires giving consideration to both reasoning complexity, and the domain in which the system will be used.

In closing, our main observations can be summarized as follows. First, from a normative systems perspective, we know there are many many logics of prioritized rules/norms, and we consider only three here. The choice we make (Greedy, Reduction and Optimization) may seem arbitrary. However, many other examples, in particular detachment procedures not satisfying defeasible deontic detachment, are much easer to characterize, while the three throughput variants of Greedy, Reduction and Optimization can be derived from the existing results. Furthermore, these three alternatives display quite diverse behavior, and are illustrative of the various kind of approaches around. 

Second, the results we present are interesting and promising, but the work on representing prioritized rules/norms using argumentation has only begun, and there are many open issues. In particular, the restriction to totally ordered systems must be relaxed in future work.

Third, given the large number of possibilities and the vast existing literature on normative rules/norms, a different methodology is needed for dealing with prioritized rules/norms in formal argumentation. 

Finally, one may wonder why our results have not been shown before, given the long standing discussion on weakest vs last link at least since the work of Pollock \cite{pollock95cognitive}, and the central role of prioritized rules in many structured argumentation theories like ASPIC$^+$. The reason, we believe, is that it is easier to study these issues on a small fragment, like hierarchical abstract normative systems, than on a very general theory like ASPIC$^+$.

\section*{Acknowledgements}

The authors are grateful to the anonymous reviewers for their helpful comments.
The research reported in this paper was partially supported by the Fundamental Research Funds for the Central Universities of China for the project Big Data, Reasoning and Decision Making, the National Research Fund Luxembourg (FNR) under grant INTER/MOBILITY/14/8813732 for the project FMUAT: Formal Models for Uncertain Argumentation from Text, and the European Union's Horizon 2020 research and innovation programme under the Marie Sklodowska-Curie grant agreement No 690974 for the project MIREL: MIning and REasoning with Legal texts.

\bibliographystyle{splncs}

\end{document}